\documentclass{article}
\usepackage[utf8]{inputenc}
\usepackage{mathtools}
\title{Bias-Variance Trade-off in Physics-Informed Neural Networks with Randomized Smoothing for High-Dimensional PDEs}

\usepackage[a4paper,top=3cm,bottom=2cm,left=2.3cm,right=2.3cm,marginparwidth=1.75cm]{geometry}
\usepackage{graphicx}
\usepackage{verbatim}
\usepackage{hyperref}       
\usepackage{url}            
\usepackage{booktabs}       
\usepackage{amsfonts}       
\usepackage{nicefrac}       
\usepackage{microtype}      
\usepackage{amsmath}
\usepackage{multirow}
\usepackage{amsthm}
\usepackage{amssymb}
\newtheorem{theorem}{Theorem}[section]

\usepackage{subfigure}
\usepackage{times}
\usepackage{color}
\newcommand{\bx}{\boldsymbol{x}}

\graphicspath{{./Figures/}}

\begin{document}
\author{Zheyuan Hu\thanks{Equal Contribution} \ \thanks{Department of Computer Science, National University of Singapore, Singapore, 119077 (\href{mailto:e0792494@u.nus.edu}{e0792494@u.nus.edu},\href{mailto:kenji@nus.edu.sg}{kenji@nus.edu.sg})} \and Zhouhao Yang\footnotemark[1] \ \footnotemark[2] \and Yezhen Wang\footnotemark[1] \ \footnotemark[2] \and George Em Karniadakis\thanks{Division of Applied Mathematics, Brown University, Providence, RI 02912, USA (\href{mailto:george\_karniadakis@brown.edu}{george\_karniadakis@brown.edu})} \ \thanks{Advanced Computing, Mathematics and Data Division, Pacific Northwest National Laboratory, Richland, WA, United States} \and Kenji Kawaguchi\footnotemark[2]}

\date{}

\maketitle
\begin{abstract}
Physics-Informed Neural Networks (PINNs) have triggered a paradigm shift in scientific computing, leveraging mesh-free properties and robust approximation capabilities. While proving effective for low-dimensional partial differential equations (PDEs), the computational cost of PINNs remains a hurdle in high-dimensional scenarios. This is particularly pronounced when computing high-order and high-dimensional derivatives in the physics-informed loss. Randomized Smoothing PINN (RS-PINN) introduces Gaussian noise for stochastic smoothing of the original neural net model, enabling the use of Monte Carlo methods for derivative approximation, which eliminates the need for costly automatic differentiation. Despite its computational efficiency, especially in the approximation of high-dimensional derivatives, RS-PINN introduces biases in both loss and gradients, negatively impacting convergence, especially when coupled with stochastic gradient descent (SGD) algorithms. We present a comprehensive analysis of biases in RS-PINN, attributing them to the nonlinearity of the Mean Squared Error (MSE) loss as well as the intrinsic nonlinearity of the PDE itself. We propose tailored bias correction techniques, delineating their application based on the order of PDE nonlinearity. The derivation of an unbiased RS-PINN allows for a detailed examination of its advantages and disadvantages compared to the biased version. Specifically, the biased version has a lower variance and runs faster than the unbiased version, but it is less accurate due to the bias. To optimize the bias-variance trade-off, we combine the two approaches in a hybrid method that balances the rapid convergence of the biased version with the high accuracy of the unbiased version. In addition to methodological contributions, we present an enhanced implementation of RS-PINN. Extensive experiments on diverse high-dimensional PDEs, including Fokker-Planck, Hamilton-Jacobi-Bellman (HJB), viscous Burgers', Allen-Cahn, and Sine-Gordon equations, illustrate the bias-variance trade-off and highlight the effectiveness of the hybrid RS-PINN. Empirical guidelines are provided for selecting biased, unbiased, or hybrid versions, depending on the dimensionality and nonlinearity of the specific PDE problem.
\end{abstract}

\section{Introduction}
Physics-Informed Neural Networks (PINNs) \cite{raissi2019physics} have revolutionized the scientific computing field thanks to their mesh-free properties, robust approximation, rapid convergence, and strong generalization capabilities \cite{jagtap2020adaptive,kawaguchi2016deep,kawaguchi2017generalization}. 
Although PINNs have proven effective in solving many low-dimensional PDEs, the computational cost remains significant in high-dimensional scenarios. This is particularly evident when calculating high-order, high-dimensional derivatives of the neural network model concerning its inputs, especially in the context of computing the physics-informed loss. The intrinsic value of unlocking the potential of PINN lies in their mesh-free training, which allows them to overcome the curse-of-dimensionality. The capability to address high-dimensional PDE problems holds immense significance, offering substantial value in addressing a myriad of practical applications, e.g., the Hamilton-Jacobi-Bellman (HJB) equation in control theory, the Black-Scholes equation in mathematical finance, and the Schr\"{o}dinger equation in quantum physics. 

Among the variants of PINNs, the randomized smoothing PINN is a promising approach, see \cite{he2023learning}.
Specifically, Randomized Smoothing PINN (RS-PINN) \cite{he2023learning} introduces Gaussian noise for stochastic smoothing of a neural network model, allowing its derivatives with respect to inputs to be expressed as expectations. This enables the model and its derivatives to be approximated using Monte Carlo methods, circumventing the challenges associated with high-order, high-dimensional derivatives, where computation by automatic differentiation can be prohibitively expensive.
While RS-PINN presents an efficient backpropagation-free method for PINN parameterization and training, its reliance on Monte Carlo to approximate expectations introduces biases in both its loss and gradients. Given that RS-PINNs are commonly coupled with stochastic gradient descent (SGD) algorithms, such as Adam \cite{kingma2014adam}, the unbiasedness of stochastic gradients is crucial for the convergence of RS-PINNs. In fact, the unbiasedness of stochastic gradients constitutes the fundamental assumption for the convergence of SGD. This phenomenon hinders the model from converging to the optimal point, making RS-PINNs perform much worse than vanilla PINNs based on automatic differentiation.

In this paper, we conduct an in-depth analysis of the sources of bias in RS-PINNs, stemming from the nonlinearity of the commonly used Mean Squared Error (MSE) loss in PINN as well as the inherent nonlinearity of the PDE itself. We demonstrate how to correct these biases separately and extend our formulation to the specific order of PDE nonlinearity. Specifically, for nonlinear PDEs with various orders, we illustrate how their biases can be corrected differently. Overall, the essence of bias correction lies in the re-sampling using distinct Gaussian random samples.
After derivation of the unbiased version of RS-PINN, we analyze the advantages and disadvantages of biased and unbiased versions. To this end, the biased version exhibits faster running speed, while under the same sample size, the unbiased version tends to have a larger variance, leading to the exploration of the bias-variance trade-off. We discuss various scenarios where the unbiased/biased version might perform better and propose a combination of both to achieve convergence speeds comparable to the biased version and the accuracy of the unbiased version. Concretely, in the optimization’s initial stages, we use the biased version to rapidly converge the model to a reasonably good point. Once the loss of the biased version ceases to decrease, we transition to the unbiased version for fine-tuning. In addition to our methodological contributions, we also present an improved implementation of RS-PINN.

Finally, through extensive experiments on several high-dimensional PDEs, including the linear Fokker-Planck PDE, the nonlinear HJB equation, the viscous Burgers' equation, the Allen-Cahn equation, and the Sine-Gordon equation in different high-dimensional scenarios, we illustrate the bias-variance trade-off and how the hybrid version adeptly assimilates the strengths and weaknesses of both versions to achieve optimal results. Through experiments, we also empirically provide guidelines for the usage of biased, unbiased, and hybrid versions, which are dependent on the dimensionality as well as the nonlinearity of the PDE problem.

The rest of this paper is arranged as follows. We discuss related work in Section 2. We provide an introduction to RS-PINNs in Section 3. Then, we introduce the bias correction techniques and our main algorithms in Section 4. Computational experiments are conducted in Section 5, and we conclude the paper in Section 6.

\section{Related Work}
\textbf{Randomized Smoothing}.
Randomized smoothing was initially proposed to tackle the adversarial robustness problem via certified robustness in neural networks, especially in image classification \cite{pmlr-v97-cohen19c,lecuyer2019certified}.
Later, it was extended to train PINNs without stacked backpropagation \cite{he2023learning}, which avoids the huge computational cost, especially in high-dimensional PDEs. 
The generalization property of the RS-PINN can be understood via the information bottleneck theory \cite{icml2023kzxinfodl}. It is expected to improve the generalization property by reducing the mutual information between the input and the hidden layer via the additionally injected noise.
More recently, randomized smoothing has also been applied to backpropagation-free federated learning \cite{feng2023does}.

\textbf{Backpropagation-Free PINNs}. In \cite{CHIU2022114909}, the authors proposed a coupled-automatic-numerical PINN (CAN-PINN), which combines automatic differentiation (AD) and numerical differentiation (ND) to become both accurate like AD and efficient like ND.
The authors of \cite{lv2021hybrid} proposed a hybrid finite difference PINN (HFD-PINN), which adopts AD for smooth scales and a weighted essentially non-oscillatory (WENO) scheme to capture discontinuity.
Fractional PINN (fPINN) \cite{pang2019fpinns} was proposed to solve fractional advection-diffusion equations, where the ND scheme is adopted for fractional differentiation.
The Deep Galerkin method (DGM) \cite{sirignano2018dgm} proposed a Monte-Carlo-based algorithm for fast second-order derivatives calculation.
Taylor mode AD \cite{bettencourt2019taylormode} was proposed to mitigate the exponential computational burden with increasing order of derivatives, which is currently available in the Jax framework.

\textbf{High-Dimensional PDE Solver}. In the broader field of high-dimensional PDE solvers, numerous attempts have been made.
The authors of \cite{wang20222} proved the importance of $L^\infty$ loss in solving high-dimensional Hamilton-Jacobi-Bellman equations.
Separable PINN \cite{cho2022separable} adopts a separable structure, enabling the residual point to be a tensor product of per-dimension points, thereby expanding the batch size. However, for problems exceeding ten dimensions, memory usage becomes a significant concern.
DeepBSDE \cite{han2018solving, han2017deep} and its extensions \cite{beck2019machine, chan2019machine,henry2017deep,hure2020deep,ji2020three} are based on the classical BSDE interpretation of certain high-dimensional parabolic PDEs, and deep learning models are employed to approximate the unknowns in the formulation.
The deep splitting method \cite{beck2021deep} integrates the classical splitting method with deep learning.
FBSNN \cite{raissi2018forward} connects high-dimensional parabolic PDEs with forward-backward stochastic differential equations and adopts deep learning for approximating the unknown solution.
The multilevel Picard methods \cite{beck2020overcoming,beck2020overcoming_ac,becker2020numerical,hutzenthaler2020overcoming, hutzenthaler2021multilevel} are a nonlinear extension of Monte Carlo that can provably solve parabolic PDEs under certain constraints. The authors of 
\cite{wang2022tensor, wang2022solving} proposed tensor neural networks, which adopt a separable structure for cheap numerical integration in solving high-dimensional Schr\"{o}dinger equations.
More recently, SDGD \cite{hu2023tackling} was proposed to sample the dimension in PDEs for scaling up and speeding up high-dimensional PINNs.

\textbf{Physics-Informed Machine Learning}. The algorithmic concepts in this paper are based on Physics-Informed Machine Learning \cite{karniadakis2021physics}, especially PINNs \cite{raissi2019physics}, which utilize neural networks as surrogate models for PDE solution approximation and optimize the boundary and residual losses, which are theoretically grounded to help the neural network model discover the correct solution \cite{hu2021extended,mishra2020estimates,shin2020convergence}.

\section{Preliminary}
\subsection{Physics-Informed Neural Networks (PINNs)}
This paper focuses on solving the following partial differential equations (PDEs) defined on a domain $\Omega \subset \mathbb{R}^d$:
\begin{equation}\label{eq:PDE}
\begin{aligned}
\mathcal{B}u(\bx)=B(\bx) \ \text{on}\ \Gamma, \qquad
\mathcal{L}u(\bx)=g(\bx) \ \text{in}\ \Omega,
\end{aligned}
\end{equation}
where $\mathcal{L}$ and $\mathcal{B}$ are the differential operators for the residual condition in $\Omega$ and for the boundary/initial condition on $\Gamma$.
PINNs \cite{raissi2019physics} is a neural network-based PDE solver via minimizing the following boundary and residual loss functions. Concretely, given the boundary points $\{\bx_{b,i}\}_{i=1}^{n_b} \subset \Gamma$ and the residual points $\{\bx_{r,i}\}_{i=1}^{n_r} \subset \Omega$, the PINN loss is composed of the mean square error in the residual and on the boundary:
\begin{equation}
\begin{aligned}
\mathcal{L}(\theta) &= \lambda_b \mathcal{L}_b(\theta) + \lambda_r \mathcal{L}_r(\theta)\\
&=\frac{\lambda_b}{n_b}\sum_{i=1}^{n_b} {|\mathcal{B}u_{\theta}(\bx_{b,i})-B(\bx_{b,i})|}^2 + \frac{\lambda_r}{n_r}\sum_{i=1}^{n_r} {|\mathcal{L}u_{\theta}(\bx_{r,i})-g(\bx_{r,i})|}^2,
\end{aligned}
\end{equation}
where $\lambda_b$ is the weight for the boundary loss while $\lambda_r$ is that for the residual loss.

\subsection{Randomized Smoothing PINNs}
He et al. \cite{he2023learning} proposed the randomly smoothed neural network structure for backpropagation-free PINN computation:
\begin{equation}
u(\bx; \theta) = \mathbb{E}_{\delta \sim \mathcal{N}(0, \sigma^2\boldsymbol{I})}f(\bx+\delta;\theta),
\end{equation}
where $f(\bx;\theta)$ is a vanilla neural network parameterized by $\theta$; $u(\bx;\theta)$ is the corresponding smoothed version of the network $f(\bx;\theta)$ serving as the surrogate model in PINNs.

The derivative of $u(\bx;\theta)$ can be analytically computed without backpropagation. For instance, its gradient, Laplacian, and Hessian with respect to the input $\bx$ can be written as follows (see He et al. \cite{he2023learning}):
\begin{equation}
\nabla_{\bx} u(\bx; \theta) = \mathbb{E}_{\delta \sim \mathcal{N}(0, \sigma^2\boldsymbol{I})}\left[\frac{\delta}{\sigma^2}f(\bx+\delta;\theta)\right].
\end{equation}
\begin{equation}
\Delta_{\bx} u(\bx;\theta) = \mathbb{E}_{\delta \sim \mathcal{N}(0, \sigma^2\boldsymbol{I})}\left[\frac{\Vert\delta\Vert^2-\sigma^2d}{\sigma^4}f(\bx+\delta;\theta)\right].
\end{equation}
\begin{equation}
\text{Hess}_{\bx}u(\bx;\theta) = \mathbb{E}_{\delta \sim \mathcal{N}(0, \sigma^2\boldsymbol{I})}\left[\frac{\delta\delta^{\mathrm{T}}-\sigma^2d}{\sigma^4}f(\bx+\delta;\theta)\right].
\end{equation}
Here $\nabla_{\bx}, \Delta_{\bx}, \text{Hessian}_{\bx}$ are the gradient, Laplacian, and Hessian operators with respect to the input $\bx$, respectively.
All of them can be simulated by Monte Carlo sampling for the expectation estimator to calculate the derivatives without the expensive automatic differentiation, and then the PINN loss is used to solve the PDE.

He et al. \cite{he2023learning} also introduced a corresponding variance reduction form, as the entire derivative estimation involves Monte Carlo estimation of expectations, which contains certain variance. Specifically, the variance reduction is related to control variate and antithetic variable
method, whose ultimate forms are similar to the numerical differentiation:
\begin{equation}
\nabla_{\bx} u(\bx; \theta) = \mathbb{E}_{\delta \sim \mathcal{N}(0, \sigma^2\boldsymbol{I})}\left[\frac{\delta}{2\sigma^2}\left(f(\bx+\delta;\theta) - f(\bx-\delta;\theta)\right)\right].
\end{equation}
\begin{equation}
\Delta_{\bx} u(\bx;\theta) = \mathbb{E}_{\delta \sim \mathcal{N}(0, \sigma^2\boldsymbol{I})}\left[\frac{\Vert\delta\Vert^2-\sigma^2d}{2\sigma^4}\left(f(\bx+\delta;\theta)+f(\bx-\delta;\theta)-2f(\bx;\theta)\right)\right].
\end{equation}
\begin{equation}
\text{Hess}_{\bx}u(\bx;\theta) = \mathbb{E}_{\delta \sim \mathcal{N}(0, \sigma^2\boldsymbol{I})}\left[\frac{\delta\delta^{\mathrm{T}}-\sigma^2d}{2\sigma^4}\left(f(\bx+\delta;\theta)+f(\bx-\delta;\theta)-2f(\bx;\theta)\right)\right].
\end{equation}

The essence of the randomized smoothing PINN lies in transforming the derivatives and model inference into an expectation. Especially in high-dimensional scenarios, this approach is highly cost-effective since computing the Hessian and other high-order derivatives of a complicated neural network in high dimensions is prohibitively expensive.
The Monte Carlo sampling method for estimating expectations serves as a powerful tool to combat the curse-of-dimensionality, especially when combined with the mesh-free PINN approach. This synergy positions it as a formidable tool for addressing high-dimensional PDEs.

Specifically, given a sample size of $K \in \mathbb{Z}^+$ in Monte Carlo, we can approximate the expectations above as follows.
\begin{equation}\label{eq:RS_MC_1}
\widehat{u}(\bx;\theta;\delta) := \frac{1}{K}\sum_{i=1}^Kf(\bx+\delta_i;\theta) \approx \mathbb{E}_{\delta \sim \mathcal{N}(0, \sigma^2\boldsymbol{I})}f(\bx+\delta;\theta) = u(\bx; \theta),
\end{equation}
where $\delta = \{\delta_i\}_{i=1}^K$ are $K$ $i.i.d.$ Gaussian samples and $\widehat{u}(\bx;\theta;\delta)$ is the Monte-Carlo-based estimation of the exact $u(\bx;\delta)$ on the point $\bx$. Similarly, for the gradient, Laplacian, and Hessian, we have the following Monte-Carlo-based estimators, which are then substituted into the PINN loss for optimization:
\begin{equation}\label{eq:RS_MC_2}
\widehat{\nabla_{\bx} u}(\bx; \theta;\delta) :=  \frac{1}{K}\sum_{i=1}^K\left[\frac{\delta}{2\sigma^2}\left(f(\bx+\delta_i;\theta) - f(\bx-\delta_i;\theta)\right)\right].
\end{equation}
\begin{equation}\label{eq:RS_MC_3}
\widehat{\Delta_{\bx} u}(\bx;\theta;\delta) = \frac{1}{K} \sum_{i=1}^K \left[\frac{\Vert\delta_i\Vert^2-\sigma^2d}{2\sigma^4}\left(f(\bx+\delta_i;\theta)+f(\bx-\delta_i;\theta)-2f(\bx;\theta)\right)\right].
\end{equation}
\begin{equation}\label{eq:RS_MC_4}
\widehat{\text{Hess}_{\bx}u}(\bx;\theta;\delta) = \frac{1}{K} \sum_{i=1}^K \left[\frac{\delta_i\delta_i^{\mathrm{T}}-\sigma^2d}{2\sigma^4}\left(f(\bx+\delta_i;\theta)+f(\bx-\delta_i;\theta)-2f(\bx;\theta)\right)\right].
\end{equation}

\section{Proposed Method}
In this section, we show that the original formulation of randomized smoothing PINN leads to a biased gradient. We further show that bias comes from two contributions: the nonlinear mean square error loss function in the PINN loss and the nonlinearity of PDE itself. These nonlinearities disrupt the linearity of the mathematical expectation in the RS-PINN for model inference and the model's derivatives in the PINN loss, thereby introducing bias.
Then, we correct these biases for better performance, demonstrating how the two biases affect PINN's performance differently and showing that the biased (unbiased) version has a lower (higher) gradient variance and that the biased version runs faster than the unbiased one per epoch.
Hence, we finally combine them to propose a hybrid version, which runs as fast as the biased version and as accurate as the unbiased one.

\subsection{Bias from the Mean Square Error Loss Function}
In this subsection, we illustrate the bias of RS-PINN's loss function and its gradient with respect to model parameters induced by the nonlinear mean square error loss function in the PINN loss, since its nonlinearity violates the linearity of expectation that guarantees unbiasedness.

Without loss of generality, since the inference and the derivatives of the surrogate model $u(\bx;\theta)$ can all be written as an expectations, let us use the boundary loss to demonstrate the first bias from the nonlinearity of the mean square error loss function, which includes the following expectation related to model inference:
\begin{equation}
u(\bx; \theta) = \mathbb{E}_{\delta \sim \mathcal{N}(0, \sigma^2\boldsymbol{I})}f(\bx+\delta;\theta).
\end{equation}
The boundary loss on a boundary point $\bx$ is:
\begin{equation}
L_b(\theta) = \left(u(\bx; \theta) - g(\bx)\right)^2= \left(\mathbb{E}_{\delta \sim \mathcal{N}(0, \sigma^2\boldsymbol{I})}f(\bx+\delta;\theta) - g(\bx)\right)^2, 
\end{equation}
where $g(\bx)$ is the given boundary condition.
He et al. \cite{he2023learning} approximate the boundary loss by Monte Carlo:
\begin{equation}
{L}^{(0)}_b(\theta) = \left(\widehat{u}(\bx; \theta;\delta) - g(\bx)\right)^2= \left(\frac{1}{K}\sum_{i=1}^Kf(\bx+\delta_i;\theta) - g(\bx)\right)^2.
\end{equation}
Although $\widehat{u}(\bx; \theta;\delta)$ is an unbiased estimator of $u(\bx; \theta)$, due to the nonlinear quadratic form of the mean square error loss function, the expectation of the loss function ${L}_b^{(0)}(\theta)$ is not the true loss $L_b(\theta)$:
\begin{equation}
\begin{aligned}
\mathbb{E}_{\delta}\left[{L}^{(0)}_b(\theta)\right] &= \mathbb{E}_{\delta}\left[\left(\frac{1}{K}\sum_{i=1}^Kf(\bx+\delta_i;\theta) - g(\bx)\right)^2 \right]\neq L_b(\theta)= \left(\mathbb{E}_{\delta \sim \mathcal{N}(0, \sigma^2\boldsymbol{I})}f(\bx+\delta;\theta) - g(\bx)\right)^2,
\end{aligned}
\end{equation}
i.e., the formulation of the loss function by He et al. \cite{he2023learning} is biased since the nonlinear mean square error loss function violates the linearity of mathematical expectation.

To correct the bias, we just need to sample independently two groups of Gaussian variables $\delta_i$ and $\delta_i'$ and compute the following debiased loss function:
\begin{equation}
{L}^{(1)}_b(\theta) = \left[\widehat{u}(\bx;\theta;\delta') - g(\bx)\right]\left[\widehat{u}(\bx;\theta;\delta) - g(\bx)\right]= \left[\frac{1}{K}\sum_{i=1}^Kf(\bx+\delta_i';\theta) - g(\bx)\right] \left[\frac{1}{K}\sum_{i=1}^Kf(\bx+\delta_i;\theta) - g(\bx)\right],  
\end{equation}
where $\delta_i \sim \mathcal{N}(0, \sigma^2\boldsymbol{I})$ and $\delta_i' \sim \mathcal{N}(0, \sigma^2\boldsymbol{I})$ are independent. Then, the derived loss function and its gradient with respect to $\theta$ for optimization are all unbiased, since we eliminate the bias introduced by nonlinearity through re-sampling, which breaks down the nonlinearity. This is summarized in the following theorem.
\begin{theorem}
\label{thm:unbiased1}
The loss ${L}^{(1)}_b(\theta)$ and its gradient with respect to $\theta$ are unbiased estimators for the exact loss $L_b(\theta)$ and its gradient with respect to $\theta$, respectively, i.e.,
\begin{equation}
\mathbb{E}_{\delta,\delta'}\left[{L}^{(1)}_b(\theta)\right] = L_b(\theta), \quad
\mathbb{E}_{\delta,\delta'}\left[\frac{\partial{L}^{(1)}_b(\theta)}{\partial \theta}\right] = \frac{\partial{L}_b(\theta)}{\partial \theta}.
\end{equation}
\end{theorem}
\begin{proof}
The proof is presented in Appendix \ref{appendix:A}
\end{proof}
Our previous discussion on boundary loss can be extended to the case of residual loss in linear PDEs. Due to the linearity of the mathematical expectation, the residual part of a linear PDE preserves the unbiased nature of Monte Carlo sampling. Therefore, the sole source of bias stems from the nonlinearity of the mean square error loss function.

Taking everything together, in this subsection, we have introduced a debiasing method for the loss function of linear PDEs in RS-PINN. The essence lies in eliminating the nonlinearity in the mean square error loss function through resampling, thereby leveraging the linearity of mathematical expectation to ensure unbiasedness. However, we note that this does not hold true for nonlinear PDEs whose residual parts violate the linearity of expectation. Below, we elucidate the additional bias introduced by the nonlinearity of PDEs.

\subsection{Bias from the PDE Nonlinearity}
In this subsection, we illustrate the bias of RS-PINN induced by the PDE nonlinearity that violates the linearity of expectation. Since nonlinearity differs in various PDEs, we commence with a brief illustration of the HJB equation. Subsequently, considering the diverse orders of nonlinearity inherent in distinct nonlinear equations, we extend our methodology to various and more general scenarios.

We take the Hamilton-Jacobi-Bellman (HJB) equation in He et al. \cite{he2023learning} as an example, whose nonlinear PDE part is
\begin{equation}
u_t = \Delta_{\bx} u - \Vert\nabla_{\bx} u(\bx)\Vert^2.
\end{equation}
To simplify the discussion and without the loss of generality since the linear PDE case has been tackled in the previous subsection, we ignore the linear term of the HJB equation and only consider the nonlinear term $\Vert\nabla_{\bx} u(\bx)\Vert^2$.
The true residual loss on a residual point $\bx$ is:
\begin{equation}
L_r(\theta) = \left(\Vert\nabla_{\bx}u(\bx; \theta) \Vert^2 - g(\bx)\right)^2 
\end{equation}
He et al. \cite{he2023learning} approximate the boundary loss by Monte Carlo:
\begin{equation}
{L}^{(0)}_r(\theta) = \left(\left\|\widehat{\nabla_{\bx}u}(\bx;\theta;\delta)\right\|^2 - g(\bx)\right)^2,
\end{equation}
where $\widehat{\nabla_{\bx}u}(\bx;\theta;\delta)$ defined in equation (\ref{eq:RS_MC_2}) is an unbiased estimator of $u(\bx; \theta)$ based on Monte-Carlo sampling.
The original formulation of He et al. \cite{he2023learning} is biased due to the nonlinearity of the mean square error loss function. However, our previous approach for linear PDEs is still biased due to the quadratic term $\Vert\cdot\Vert^2$ term in the loss due to the PDE nonlinearity, which violates the linearity of mathematical expectation.
Concretely, the loss correcting the bias from the forward and backward passes is
\begin{equation}
\begin{aligned}
{L}^{(1)}_r(\theta) &= \left(\left\|\widehat{\nabla_{\bx}u}(\bx;\theta;\delta)\right\|^2 - g(\bx)\right) \times \left(\left\|\widehat{\nabla_{\bx}u}(\bx;\theta;\delta')\right\|^2 - g(\bx)\right).
\end{aligned}
\end{equation}
Although $\widehat{u}(\bx; \theta;\delta)$ is an unbiased estimator of $u(\bx; \theta)$, an additional bias comes from the nonlinear term $\Vert \cdot\Vert^2$ due to the PDE nonlinearity $\Vert \nabla u \Vert^2$.

We can correct the bias via sampling the quadratic terms in $\Vert f\Vert ^2 = \langle f,f\rangle$ independently, using four groups of Gaussian samples $\delta, \delta', \delta'', \delta'''$:
\begin{equation}
\begin{aligned}
{L}^{(2)}_r(\theta) &=\left(\left\langle\widehat{\nabla_{\bx}u}(\bx;\theta;\delta), \widehat{\nabla_{\bx}u}(\bx;\theta;\delta')\right\rangle - g(\bx)\right) \times \left(\left\langle\widehat{\nabla_{\bx}u}(\bx;\theta;\delta''), \widehat{\nabla_{\bx}u}(\bx;\theta;\delta''')\right\rangle - g(\bx)\right).
\end{aligned}
\end{equation}
Then, the derived loss function and its gradient with respect to $\theta$ for optimization are all unbiased.
\begin{theorem}
\label{thm:unbiased2}
The loss ${L}^{(2)}_r(\theta)$ and its gradient with respect to $\theta$ are unbiased estimators for the loss $L_r(\theta)$ and its gradient with respect to $\theta$, respectively, i.e.,
\begin{equation}
\mathbb{E}_{\delta,\delta',\delta'',\delta'''}\left[{L}^{(2)}_r(\theta)\right] = L_r(\theta), \quad
\mathbb{E}_{\delta,\delta',\delta'',\delta'''}\left[\frac{\partial}{\partial \theta}{L}^{(2)}_r(\theta)\right] = \frac{\partial}{\partial \theta}L_r(\theta).
\end{equation}
\end{theorem}
\begin{proof}
The proof is presented in Appendix \ref{appendix:A}.
\end{proof}
So far, we present the bias correction technique for one nonlinear PDE, namely the HJB equation, where we correct the biases from the nonlinear mean square error loss function and the PDE nonlinearity. We will discuss general nonlinear PDEs in the next subsection.

\subsection{General Nonlinear PDEs: Order of Nonlinearity}
This subsection extends our bias correction technique to general nonlinear PDEs based on the concept of nonlinearity order.
Specifically, our idea of bias correction can be easily extended to the general $n$th order of nonlinearity.
In the context of nonlinear PDEs, the ``order of nonlinearity" refers to the highest power of the dependent function $u$ or its derivatives in the nonlinear terms of the equation. Nonlinearity in PDEs arises when the equation involves terms that are not proportional and linear to the dependent function or its derivatives. The order of nonlinearity is determined by the highest power of these nonlinear terms.
Intuitively, for the $n$th order nonlinear case, we just need to sample the $n$ terms independently using $n$ different groups of Gaussian variables to break down the nonlinearity and to make the loss function unbiased.
Below are some examples of the nonlinearity order and their corresponding biased and unbiased versions of RS-PINN, whose detailed explanation is further provided in Appendix \ref{appendix:B}.
\begin{itemize}
\item HJB equation. The previous HJB equation given by
$
u_t = \Delta_{\bx} u - \Vert\nabla_{\bx} u(\bx)\Vert^2,$
which has a nonlinearity order of two due to the $\Vert \nabla_{\bx} u \Vert^2$ term.
\item Allen-Cahn (AC) equation is given by
$u_t = \Delta_{\bx} u + u - u^3$, and
its nonlinearity stems from the term $u^3$, which is a cubic function. Therefore, the nonlinearity order of the AC equation is three.
During the model training, we must independently sample the three $u$ terms in $u^3 = u \cdot u \cdot u$ for unbiased gradients.
\item Viscous Burgers' equation is given by
$
u_t + u\sum_{i=1}^du_{\bx_i} - \nu\Delta_{\bx} u(\bx, t) = 0,
$ and
its nonlinearity stems from the term $u\sum_{i=1}^du_{\bx_i}$. Therefore, the nonlinearity order of the viscous Burgers' PDE is two.
During the model training, we must independently sample the $u$ and $\nabla_{\bx} u$ for unbiased gradients.
\item Sine-Gordon equation. However, our method cannot deal with nonlinear like $\sin (u)$ in the Sine-Gorden PDE
$
u_t = \Delta_{\bx} u + \sin (u).
$
Fortunately, we can increase the sample size $K$ in the Monte Carlo to minimize the bias.
Furthermore, correcting the bias from the nonlinearity mean square error loss does suffice for the Sine-Gordon equation, i.e., we can still correct the bias from the nonlinear mean square error loss to improve over the original formulation in He et al. \cite{he2023learning}, which is actually sufficient to obtain a low error in high dimensions.
\end{itemize}

\subsection{Bias-Variance Trade-off and the Hybrid Method}
In this subsection, we discuss the bias-variance trade-off in RS-PINN and propose a hybrid version to incorporate the advantages of both methods to achieve the best performance. We also provide guidelines for explaining when the biased/unbiased version can outperform the other, facilitating the choice of the algorithm in practical scenarios.

The unbiased version employs multiple sets of independent Gaussian samples to calculate the loss, making it slower with larger gradient variances due to more sampling and more randomness; however, it provides unbiased gradients. Specifically, while the biased version from He et al. \cite{he2023learning} requires only one set of Gaussian variables, correcting the bias from the nonlinear MSE loss functions doubles the number of independent Gaussian variable sets while correcting the additional bias from the PDE nonlinearity further increases the number of samples depending on the nonlinearity order of the PDE. For instance, a totally unbiased version of the HJB equation and the viscous Burgers' equation requires four sets, while that of the Allen-Cahn equation requires six sets.

In contrast, the biased version requires only one set of samples, resulting in faster running speed per iteration and smaller gradient variances, but the gradients are biased.
Hence, we propose a hybrid approach that combines the strengths of both methods.
In the initial optimization stages, we use the biased version to converge the model rapidly to a reasonably good point. Once the loss of the biased version ceases to decrease, we transition to the unbiased version for fine-tuning.

This theoretical analysis sheds light on the practical choice of algorithms in computational experiments. In higher dimensions, in the unbiased version by sampling more Gaussians will lead to a much larger variance.
So, it is expected that the unbiased version will have lower variance in lower dimensions and thus have better performance than the biased one.
On the other hand, the biased version will perform better in extremely high dimensions. After the convergence of the biased algorithm, we can further fine-tune it using the unbiased algorithm.

In summary, our guidelines for empirical evaluations based on the theoretical analysis are given as follows. 
In lower dimensions, where the unbiased version exhibits lower variance, its unbiased nature is crucial, allowing for a direct application of the unbiased version. However, in higher dimensions, utilizing the unbiased version directly introduces significant variance, impeding convergence. Therefore, we employ the biased version initially to converge to a reasonably good position and subsequently fine-tune with the unbiased version.

\subsection{Implementation Improvement}
Here, we conduct an analysis of He et al.'s \cite{he2023learning} approach to implementing randomized smoothing in order to identify its limitations. Subsequently, we propose two more accurate and lower-variance implementations.

Suppose that we would like to implement the second-order derivatives and the network includes both $t$ and $\bx$ for time-dependent PDEs
\begin{equation}
u(\bx, t) = \mathbb{E}_{\delta_{\bx} \sim \mathcal{N}(0, \sigma_{\bx}^2I)} \mathbb{E}_{\delta_t \sim \mathcal{N}(0, \sigma_{t}^2I)}\left[f(\bx + \delta_{\bx}, t+\delta_t)\right],
\end{equation}
where we randomly smooth $\bx$ and $t$ using Gaussian with different variance for model flexibility.
He et al. \cite{he2023learning} implement the randomized smoothing model's derivatives as
\begin{equation}
\boldsymbol{H}_{\bx}u(\bx, t) = \mathbb{E}_{\delta_{\bx} \sim \mathcal{N}(0,\sigma_{\bx}^2I)}\mathbb{E}_{\delta_t \sim \mathcal{N}(0, \sigma_{t}^2I)}\left[\frac{\delta_{\bx}\delta_{\bx}^T - \sigma_{\bx}^2I}{2\sigma_{\bx}^4}(f(\bx+\delta_{\bx}, t+\delta_t) + f(\bx-\delta_{\bx}, t-\delta_t) - 2f(\bx, t))\right].
\end{equation}
However, here we are taking the derivative with respect to $\bx$, with no relation to $t$. Nevertheless, the focus has also shifted to $t$, thereby increasing the variance and impeding convergence.

The correct approach should treat $\bx$ and $t$ as independent variables:
\begin{equation}
\begin{aligned}
\boldsymbol{H}_{\bx}u(\bx, t) &=\mathbb{E}_{\delta_{\bx} \sim \mathcal{N}(0,\sigma_{\bx}^2I)}\left[\frac{\delta_{\bx}\delta_{\bx}^T - \sigma_{\bx}^2I}{2\sigma_{\bx}^4}\mathbb{E}_{\delta_t \sim \mathcal{N}(0, \sigma_{t}^2I)}\left[f(\bx+\delta_{\bx}, t+\delta_t) + f(\bx-\delta_{\bx}, t+\delta_t) - 2f(\bx, t+\delta_t)\right]\right]\\
&= \mathbb{E}_{\delta_{\bx} \sim \mathcal{N}(0,\sigma_{\bx}^2I)}\mathbb{E}_{\delta_t \sim \mathcal{N}(0, \sigma_{t}^2I)}\left[\frac{\delta_{\bx}\delta_{\bx}^T - \sigma_{\bx}^2I}{2\sigma_{\bx}^4}(f(\bx+\delta_{\bx}, t+\delta_t) + f(\bx-\delta_{\bx}, t+\delta_t) - 2f(\bx, t+\delta_t))\right].
\end{aligned}
\end{equation}
Another valid implementation approach is to treat $\bx$ and $t$ as a unified entity and apply the same Gaussian noise smoothing. Then, based on the index, select the model's derivatives concerning both $\bx$ and $t$. However, this method compromises the model's flexibility since the PDE exhibits an asymmetry between $\bx$ and $t$. Therefore, a more reasonable approach is to model them separately.

\section{Computational Experiments}
In our computational experiments, for linear equations (Fokker-Planck PDEs in Subsection \ref{subsec:fp}), we will use ``biased'' to denote the biased version and ``unbiased'' to denote the unbiased version by correcting the bias from the MSE loss. For nonlinear PDEs in the rest of the subsections, we will use ``biased'' as before, and ``unbiased1'' to denote the unbiased version by correcting the bias from the MSE loss solely, and ``unbiased2'' to denote the unbiased version by correcting the two biases from the MSE loss and the PDE nonlinearity. The detailed mathematical formulas for the losses are presented in Appendix \ref{appendix:B}.

\subsection{Isotropic and Anisotropic Linear Fokker-Planck PDEs}\label{subsec:fp}
The isotropic linear Fokker-Planck (heat) PDE is
\begin{equation}
\begin{aligned}
&u_t = \frac{1}{2}\Delta_{\bx} u - \sum_{i=1}^d u_{\bx_i}. \quad \bx \in \mathbb{R}^d, t \in [0, 1].\\
&u(\bx, t=0) = \Vert\bx\Vert^2. \quad \bx \in \mathbb{R}^d,
\end{aligned}
\end{equation}
associated with the initial condition at $t = 0$.
Its exact solution is
\begin{equation}
u(\bx, t) = \Vert \bx -  t\Vert^2 + dt.
\end{equation}
Since this PDE corresponds to a Brownian motion with shift, We sample training residual points and test points based on the SDE trajectory:
\begin{equation}
t \sim \text{Unif}(0, 1), \bx \sim \mathcal{N}(t, 2 - t).
\end{equation}
The anisotropic linear Fokker-Planck (heat) PDE is
\begin{equation}
\begin{aligned}
&u_t = \frac{1}{2}\Delta_{\bx} u - \sum_{i=1}^d \boldsymbol{\mu}_i u_{\bx_i}. \quad \bx \in \mathbb{R}^d, t\in[0,1].\\
&u(\bx,t=0) = \Vert\bx\Vert^2. \quad \bx \in \mathbb{R}^d,
\end{aligned}
\end{equation}
associated with an initial condition at $t=0$. Its solution is
\begin{equation}
u(\bx, t) = \Vert \bx -  \boldsymbol{\mu}t\Vert^2 + dt,
\end{equation}
where $\boldsymbol{\mu}_i \sim \mathcal{N}(1, 1)$ for all dimensions $i$ and $\boldsymbol{\mu} \in \mathbb{R}^d$.
This example is designed to show that RS-PINN can deal with anisotropic problems. Since this PDE corresponds to a Brownian motion with shift, We sample training residual points and test points based on the SDE trajectory:
\begin{equation}
t \sim \text{Unif}(0, 1), \bx \sim \mathcal{N}(\boldsymbol{\mu}t, (2 - t)\cdot\boldsymbol{I}_{d\times d}).
\end{equation}
For both isotropic and anisotropic FP PDEs, we randomly sample 100 residual points at each iteration and 20K fixed test points based on the SDE trajectory. The sample size in the RS-PINN is $K = 1024$, and the variance of Gaussian noise is $\sigma=1e-2$, with a backbone network with 4 layers and 128 hidden units, which is trained by an Adam optimizer \cite{kingma2014adam} with 1e-3 (10, 100, 1K dimension) or 1e-4 (10K dimension) initial learning rate which decays exponentially with coefficient 0.9995 for 10K epochs. We use the boundary augmentation given by the following model output to satisfy the initial condition automatically \cite{lu2021physics}:
\begin{align}
u^{\text{RS}}_\theta(\bx) = u_\theta(\bx, t) t + \Vert\bx\Vert^2,
\end{align}
where $u_\theta(\bx)$ is the randomized smoothing neural network and $u^{\text{RS}}_\theta(\bx)$ is the boundary-augmented model. We repeat our experiment 5 times with 5 independent random seeds. We test RS-PINN with biased, unbiased, and hybrid versions for the 10, 100, 1K, and 10K-dimensional cases. 
For the hybrid version in the isotropic problem, the transition from the biased version to the unbiased one happens in the 1500th, 1500th, 6000th, and 6000th epochs for the 10, 100, 1K, and 10K dimensional PDEs, respectively.
For the hybrid version in the anisotropic problem, the transition from the biased version to the unbiased one happens in the 1000th, 500th, 1000th, 2000th, 6000th, and 8000th epochs for the 10, 100, 250, 500, 1K, and 10K dimensional PDEs, respectively. The transition is chosen by the time when the loss of the biased algorithm ceases to decrease further.

\begin{table}[htbp]\centering
\begin{tabular}{|c|c|c|c|c|}
\hline
Isotropic FP & $10^1$D & $10^2$D & $10^3$D & $10^4$D \\ \hline
Biased & 3.846E-3 & 2.367E-2 & 1.057E-2 & {8.597E-3} \\ \hline
Unbiased & {2.244E-3} & {6.576E-3} & {1.047E-2} & {1.197E-2} \\ \hline
Hybrid & \textbf{2.238E-3} & \textbf{6.561E-3} & \textbf{1.046E-2} & \textbf{7.507E-3} \\ \hline
\end{tabular}
\caption{Results for the isotropic FP PDE.}
\label{tab:isotropic_FP}
\end{table}

The numerical results for the isotropic FP PDE are shown in Table \ref{tab:isotropic_FP}, and Figure \ref{fig:Isotropic_FP} shows convergence curves with respect to the epoch (first row) and the running time (second row). Here is the summary of the results:
\begin{itemize}
\item In lower dimension (10D, $10^2$D), the unbiased version is much better than the biased version in He et al. \cite{he2023learning} since the variance of sampling is lower in lower dimensions where the dimensionality of the samples is low, given that the main bottleneck of the unbiased version is the relatively larger variance compared to the biased version of RS-PINN.
\item However, as the dimensions goes higher ($10^3$D, $10^4$D), the biased version gets better, i.e., the unbiased version encounters huge variance in higher dimensions, whose disadvantages outweigh its benefit of unbiasedness.
\item In $10^1, 10^2, 10^3$D, the hybrid version is as good as the unbiased version.
\item In $10^4$D, the hybrid version converges well by applying the biased version first to get a relatively good convergence point, then the unbiased version is used for finetuning and gets an even better result.
\item In $10^4$D, directly applying the unbiased version will lead to huge variances preventing convergence.
\end{itemize}

\begin{figure}[htbp]
\centering
\includegraphics[scale=0.24]{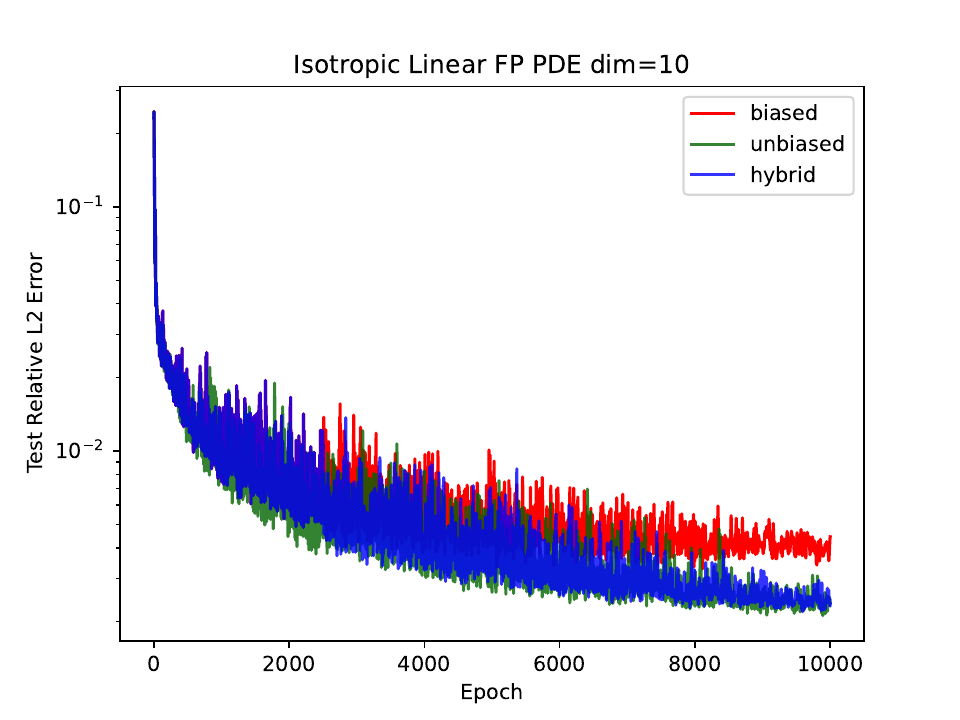}
\includegraphics[scale=0.24]{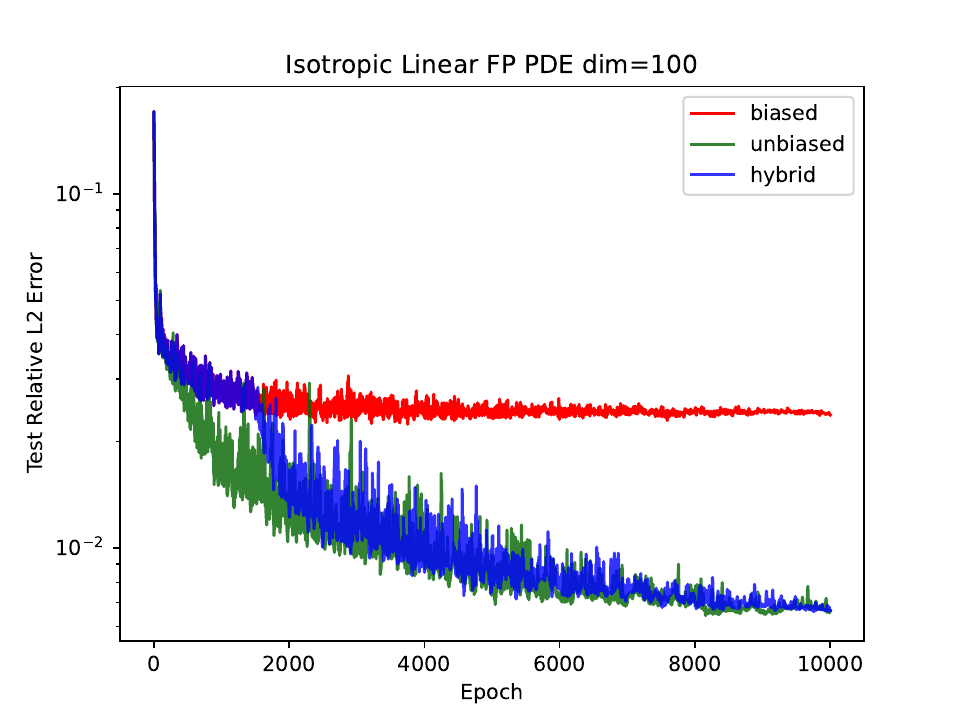}
\includegraphics[scale=0.24]{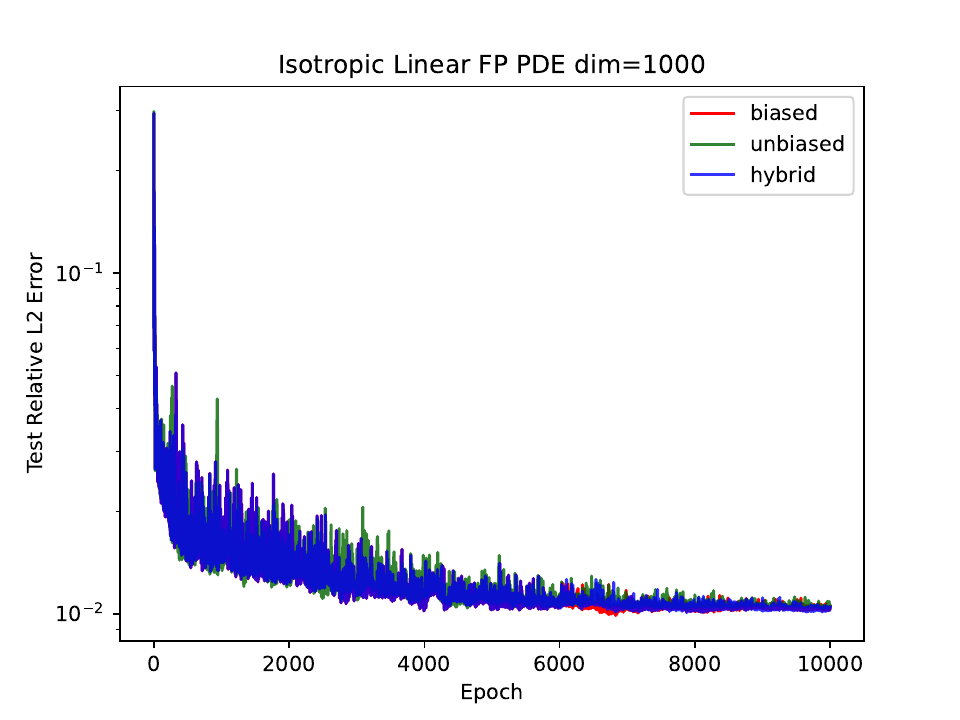}
\includegraphics[scale=0.24]{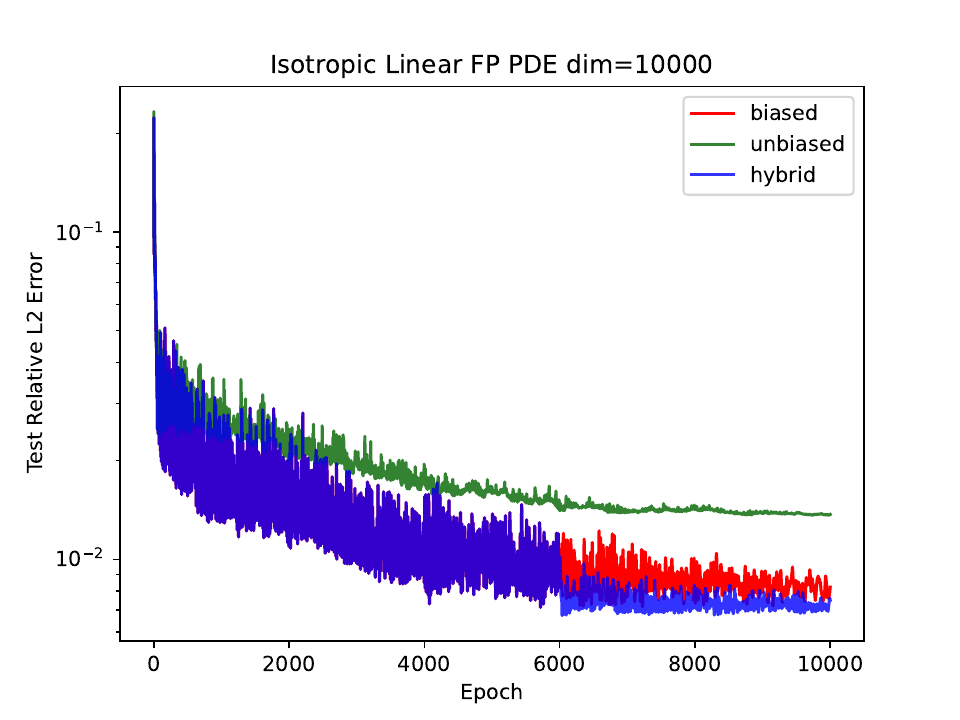}
\includegraphics[scale=0.24]{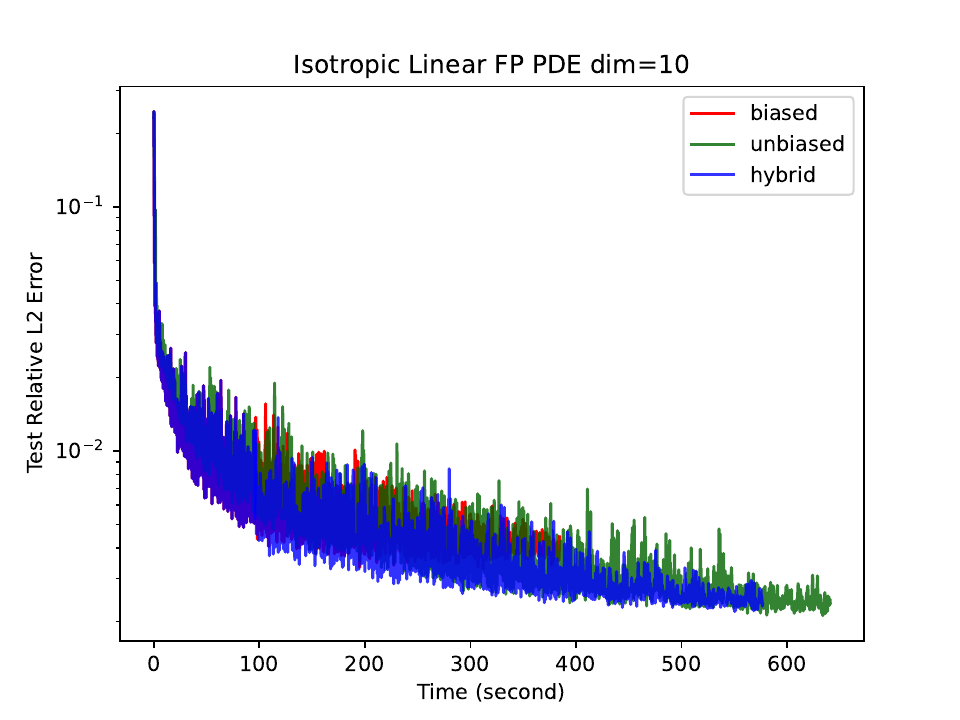}
\includegraphics[scale=0.24]{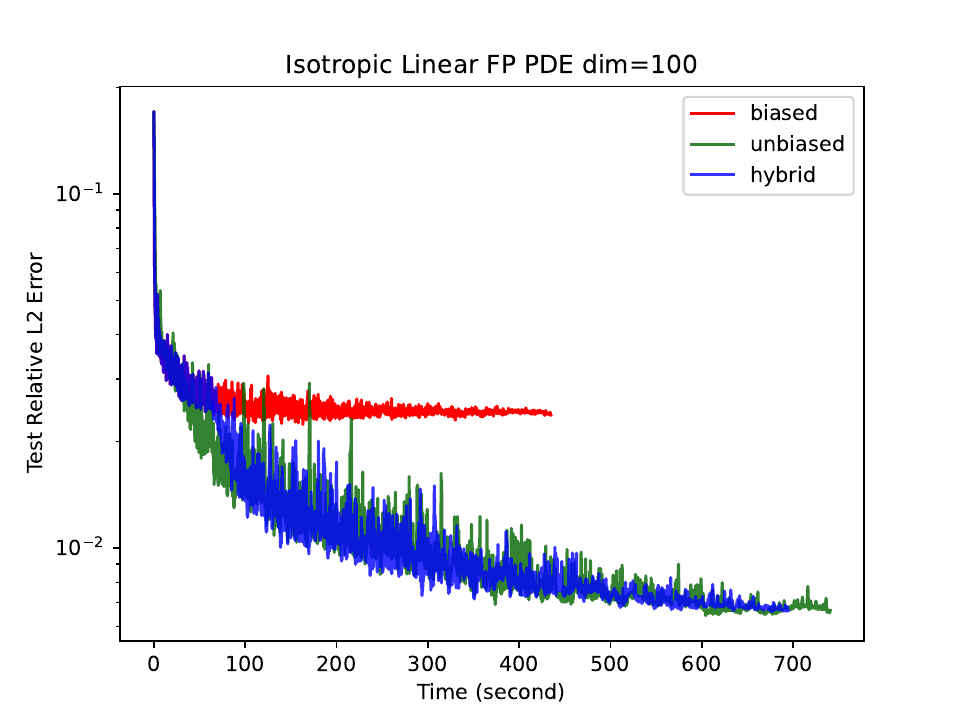}
\includegraphics[scale=0.24]{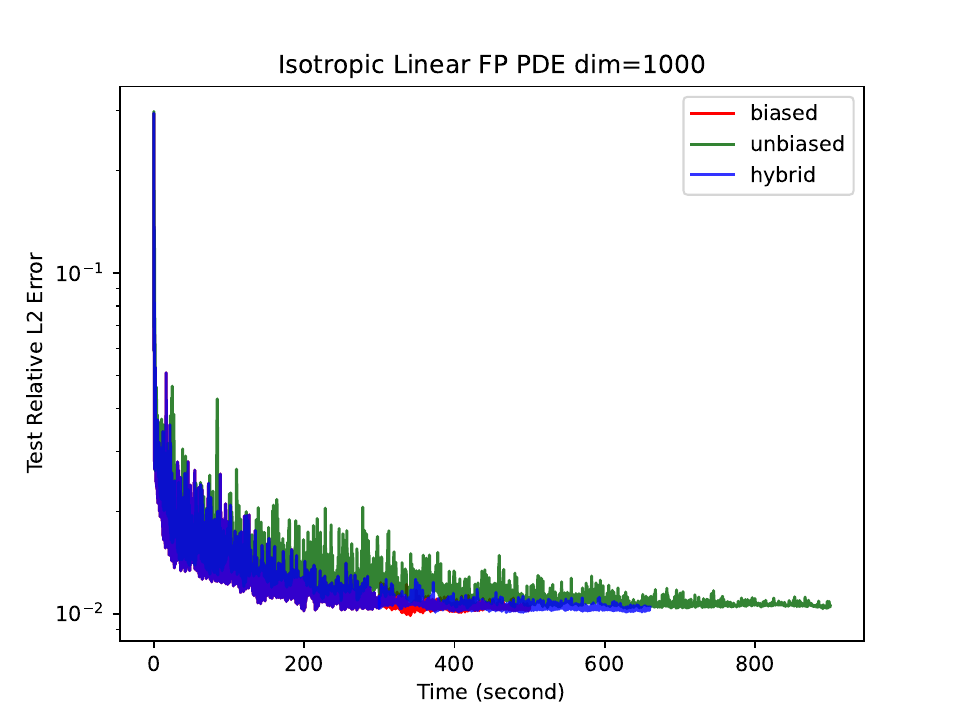}
\includegraphics[scale=0.24]{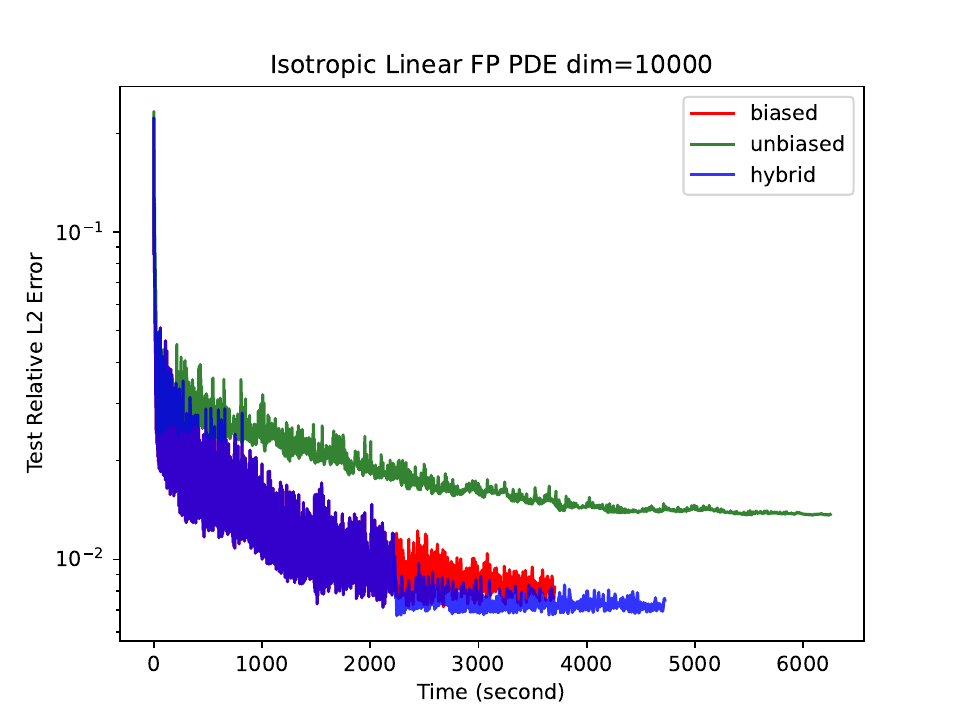}
\caption{Isotropic FP PDE: $10^1$, $10^2$, $10^3$, and $10^4$D convergence curves with respect to the epoch (first row) and the running time (second row). 
In $10^1$, $10^2$, and $10^3$D, the unbiased version is better than the biased version since the sampling variance is lower in lower dimensions where the dimensionality of the samples is low. 
Thus, the hybrid method is as good as the unbiased version thanks to the unbiased training at the second training stage and is faster than the unbiased version thanks to the biased pretraining at the early training phase.
In $10^4$D, the hybrid version converges well by applying the biased version first; then, the unbiased version is used for finetuning and getting an even more stable final convergence result. Solely applying the unbiased version will lead to huge variances preventing convergence.}
\label{fig:Isotropic_FP}
\end{figure}

The results for the anisotropic FP PDE are shown in Table \ref{tab:anisotropic_FP} while the convergence curves and loss records for the highest $10^4$D are shown in Figure \ref{fig:Anisotropic_FP_10001}. We can still observe the advantages of the unbiased version in relatively lower dimensions, and as the dimensionality increases, the biased version gradually outperforms the unbiased version. This is primarily due to the increase in variance for the unbiased version, particularly in extremely high dimensions. Notably, the results obtained by the RS-PINN remain quite stable across different dimensions, demonstrating its ability to address the dimensionality curse. Furthermore, RS-PINN exhibits competence in handling anisotropic problems. The convergence curves in the 10,000-dimensional space suggest that, after convergence is achieved with the biased version, fine-tuning with the unbiased version can lead to even better and more stable results, ultimately causing the hybrid method to outperform the rest.

\begin{table}[htbp]\centering
\begin{tabular}{|c|c|c|c|c|c|c|}
\hline
Anisotropic Linear Heat & 10D & 100D & 250D & 500D & 1,000D & 10,000D \\ \hline
Biased & 5.611E-02&	4.452E-02&	2.827E-02
& 1.935E-02	&{1.251E-02}&	{1.389E-02}
  \\ \hline
Unbiased & {1.043E-02}&	{1.215E-02}&	{1.986E-02}
& {1.846E-02}	&1.657E-02	&4.036E-02
 \\ \hline
Hybrid & \textbf{1.039E-02}&	\textbf{1.211E-02}&	\textbf{1.979E-02}& \textbf{1.840E-02}	&\textbf{1.245E-02}	&\textbf{1.342E-02}
 \\ \hline
\end{tabular}
\caption{Results for the anisotropic FP PDE.}
\label{tab:anisotropic_FP}
\end{table}

\begin{figure}[htbp]
\centering
\includegraphics[scale=0.4]{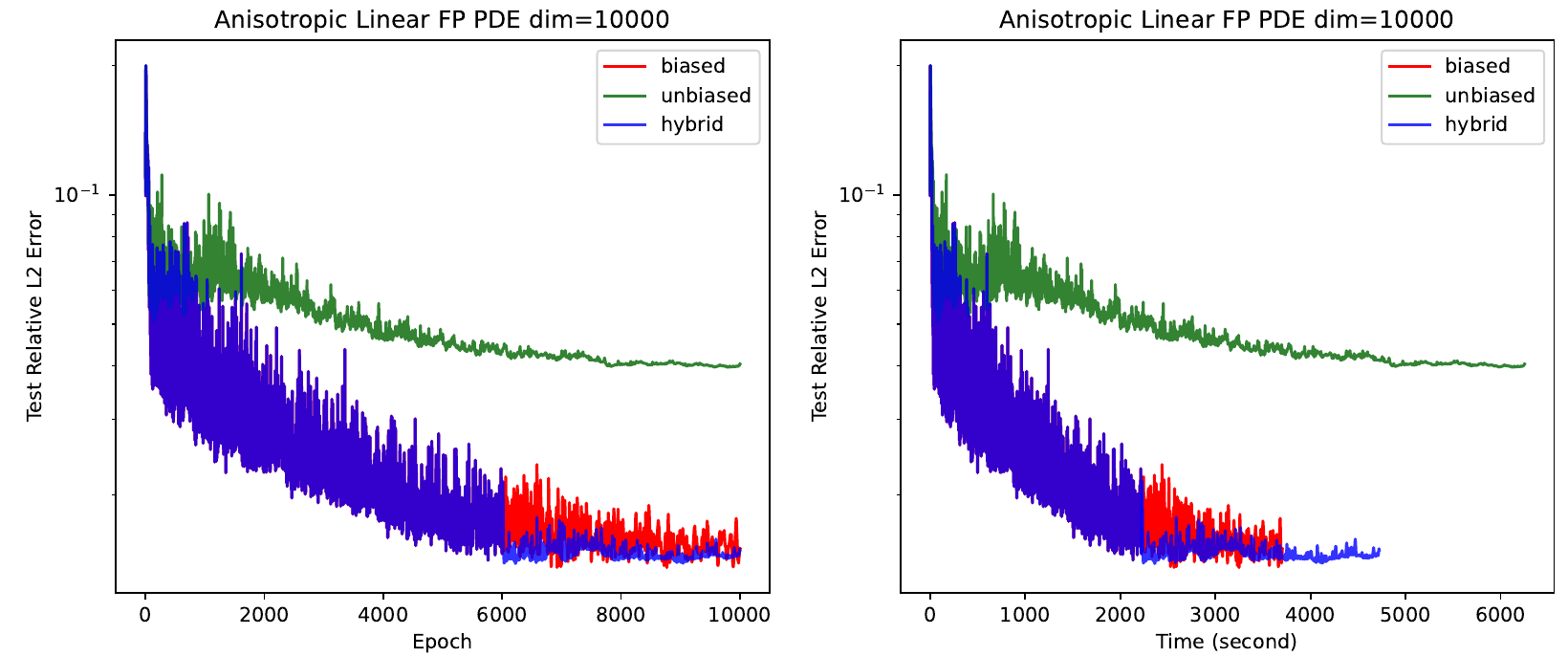}
\caption{Anisotropic FP PDE: $10^4$D convergence curves with respect to the epoch (left) and time (right). The hybrid version converges well by applying
the biased version first; then, the unbiased version is used for finetuning and getting an even more stable final convergence result. Solely applying the unbiased version will lead to huge variances preventing convergence.}
\label{fig:Anisotropic_FP_10001}
\end{figure}

\subsection{Hamilton-Jacobi-Bellman PDEs}
This section delves into the Hamilton-Jacobi-Bellman (HJB) equation, which is widely used in optimal control problems. The nonlinearity inherent in the HJB equation introduces two biases in RS-PINN. We will demonstrate the remarkable performance of the model after correcting these two biases. Correcting only one bias or adopting a completely biased approach yields suboptimal results. Additionally, we consider two different solutions to showcase the model's versatility.

Specifically, we consider the HJB equation with linear-quadratic-Gaussian (LQG) control:
\begin{equation}
\begin{aligned}
&\partial_t u(\bx, t) + \Delta_{\bx} u(\bx, t) - \Vert \nabla_{\bx} u(\bx,t) \Vert^2 = 0, \quad \bx \in \mathbb{R}^d, t\in[0,T]\\
&u(\bx,T)=g(\bx),
\end{aligned}
\end{equation} 
where $g(\bx)$ is the terminal condition to be chosen, the PDE has the solution that can be simulated by Monte Carlo for benchmarking over various initial conditions and dimensions:
\begin{equation}
u(\bx,t) = -\log\left(\int_{\mathbb{R}^d}(2\pi)^{-d/2}\exp(-\Vert \boldsymbol{y} \Vert^2/2)\exp(- g(\bx - \sqrt{2(1-t)}\boldsymbol{y}))d\boldsymbol{y}\right).
\end{equation}
We choose the following cost functions as the terminal conditions:
\begin{itemize}
\item Quadratic cost: 
\begin{equation}
g(\bx)  = \Vert\bx\Vert^2.  \quad u(\bx, t) = \frac{\Vert\bx\Vert^2}{1 + 4(T-t)} + \frac{d}{2} \log(1 + 4(T-t)).  
\end{equation}
Here, the solution can be obtained analytically.
\item Anisotropic Rosenbrock function:
\begin{equation}
g(\bx)  = \sum_{i=1}^{d/2} \left[c_{1,i}(x_{2i-1} - x_{2i})^2 + c_{2,i}x_{2i}^2\right],
\end{equation}
where $c_{1,i}, c_{2,i} \sim \text{Unif}[0, 1]$ and Monte Carlo is required for simulating the exact solution. We use $10^5$ samples for Monte Carlo.
\end{itemize}
Here is the implementation detail. For all three HJB equations, we randomly sample 1K residual points at each iteration and 20K fixed test points based on the distributions $t \sim \text{Unif}[0, 1], \bx \sim \mathcal{N}(0, \boldsymbol{I}_{d\times d})$. The sample sizes in the RS-PINN is $K = 1024$ for training and $K = 128$ for testing, and the variance of Gaussian noise is $\sigma=1e-2$, with a backbone network with 4 layers and 128 hidden units, which is trained by an Adam optimizer \cite{kingma2014adam} with 1e-3 initial learning rate which decays exponentially with coefficient 0.9995 for 10K epochs. We use the boundary augmentation given by the following model output to satisfy the terminal condition automatically \cite{lu2021physics}:
\begin{align}
u^{\text{RS}}_\theta(\bx) = u_\theta(\bx, t) t + g(\bx), 
\end{align}
where $u_\theta(\bx)$ is the randomized smoothing neural network and $u^{\text{RS}}_\theta(\bx)$ is the boundary-augmented model. We repeat our experiment 5 times with 5 independent random seeds.

The computational results for the three HJB equations are shown in Table \ref{tab:HJB}, and the convergence curves for HJB with quadratic cost are shown in Figure \ref{fig:HJB}.

\begin{table}[htbp]\centering
\begin{tabular}{|c|c|c|c|c|}
\hline
HJB (Quadratic Cost) & $10$D & $20$D & 30D & $40$D \\ \hline
Biased & 7.423E-2 & 1.882E-1 & 2.486E-1 & 3.628E-1 \\ \hline
Unbiased1 & 3.896E-2 & 1.281E-1 & 2.332E-1 & 3.204E-1 \\ \hline
Unbiased2 & \textbf{1.415E-2} & \textbf{2.572E-2} & \textbf{2.644E-2} & \textbf{4.223E-2} \\ \hline
\hline
HJB (Anisotropic Cost) & $10$D & $20$D & 30D & $40$D \\ \hline
Biased & 9.361E-2 & 1.820E-1 & 3.305E-1 & 3.998E-1 \\ \hline
Unbiased1 & 7.783E-2 & 2.058E-1 & 3.652E-1 & 4.379E-1 \\ \hline
Unbiased2 & \textbf{6.909E-2} & \textbf{6.137E-2} & \textbf{1.112E-1} & \textbf{1.417E-1} \\ \hline
\end{tabular}
\caption{Results for the HJB equation: Unbiased2 that corrects all the biases from the mean square error loss function and the PDE nonlinearity performs the best in all dimensions and in different settings with various cost functions.}
\label{tab:HJB}
\end{table}

In most cases of the HJB equation, the model that corrects both biases (unbiased2) performs the best. Following this, the model that corrects one bias in the MSE loss function (unbiased1) shows the next best performance, while the completely biased model performs the worst. This highlights the correctness of our analysis regarding the two biases introduced by analyzing nonlinear equations and underscores the improvement achieved by bias correction. 
These HJB equations are not particularly high-dimensional, so our theoretical analysis suggests that using the unbiased version in this scenario is better than the biased one. This is because the former's variance won't be substantial in cases that are not highly dimensional. Lastly, this problem is difficult because we do not have a boundary condition but deal with the unbounded domain, and the PDE solution quickly diverges to infinity when $\bx$ tends to infinity, which makes the model lack information at infinity. Efficient PINN-based algorithms for such HJB equation on unbounded domains are still open questions in the literature. Here,  we focus on comparing the performances of the biased and unbiased versions.

\begin{figure}[htbp]
\centering
\includegraphics[scale=0.24]{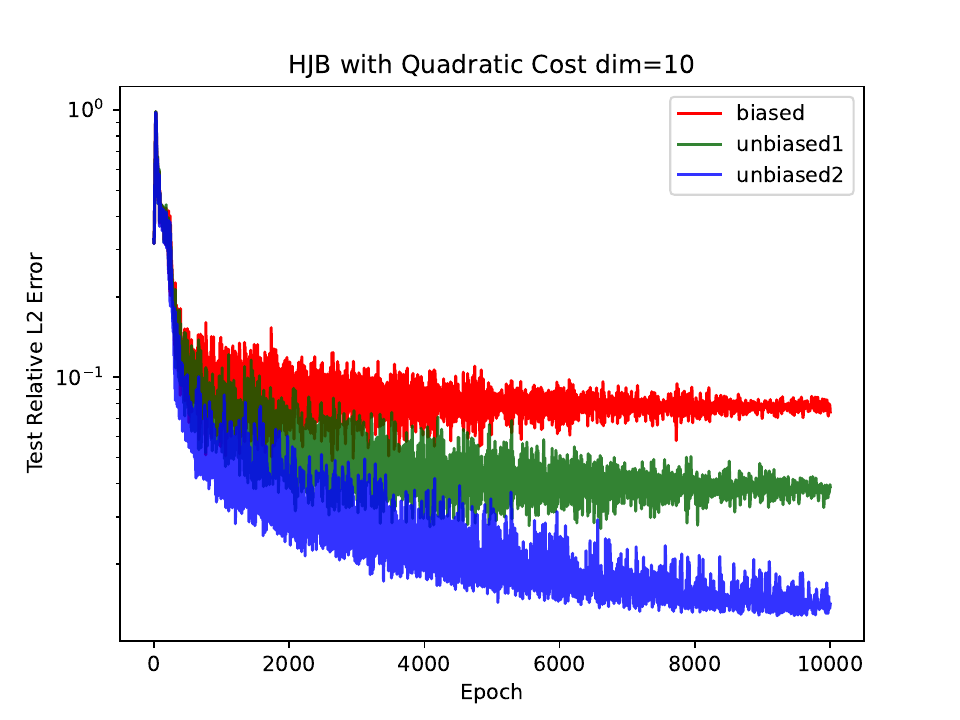}
\includegraphics[scale=0.24]{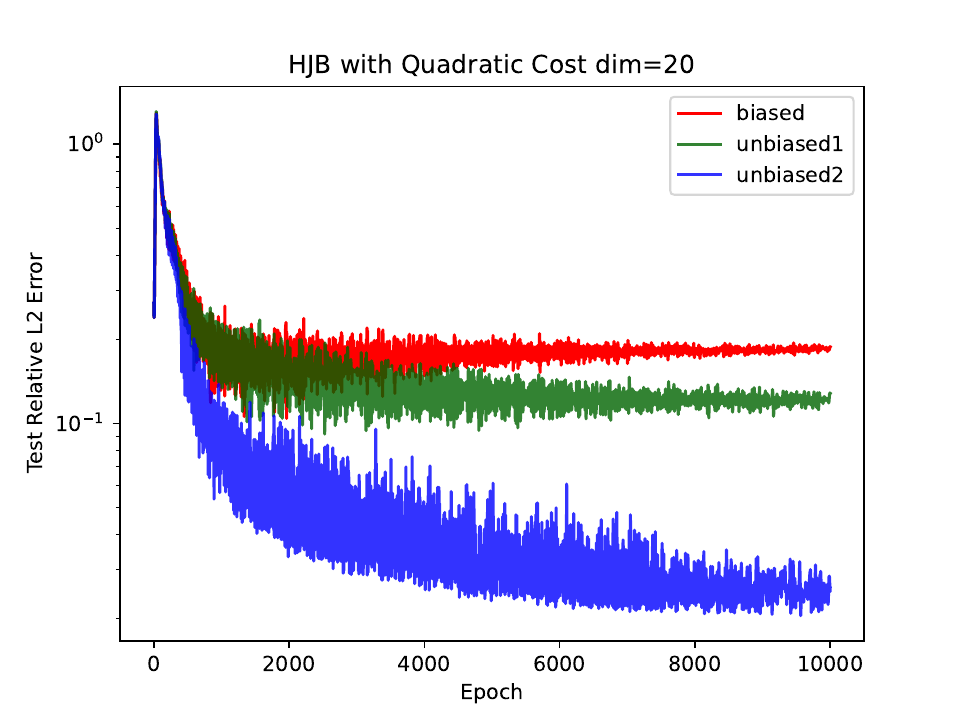}
\includegraphics[scale=0.24]{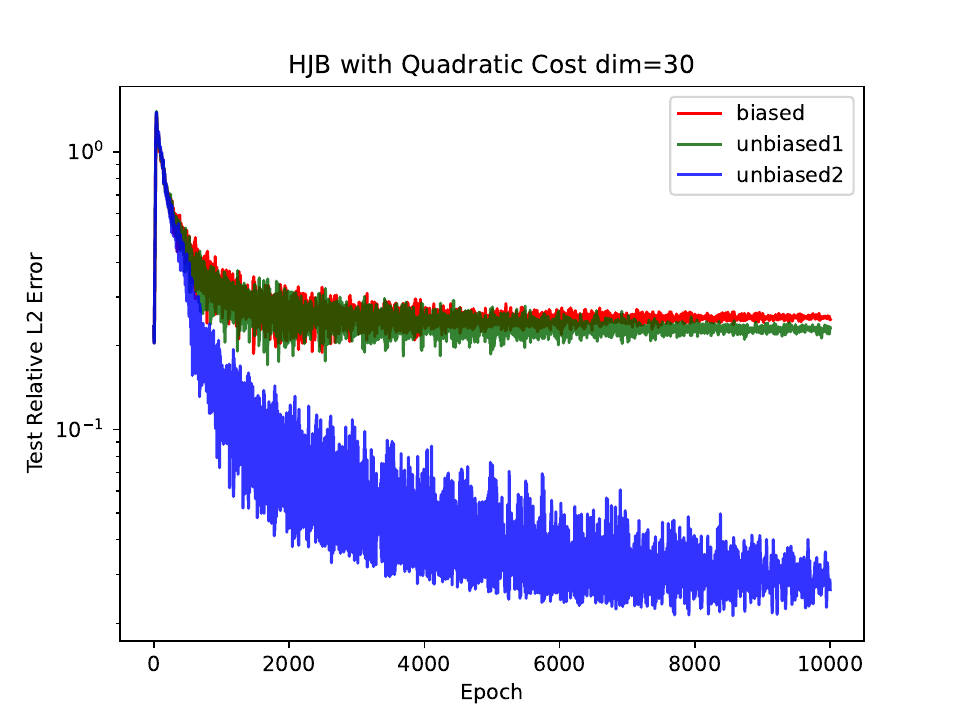}
\includegraphics[scale=0.24]{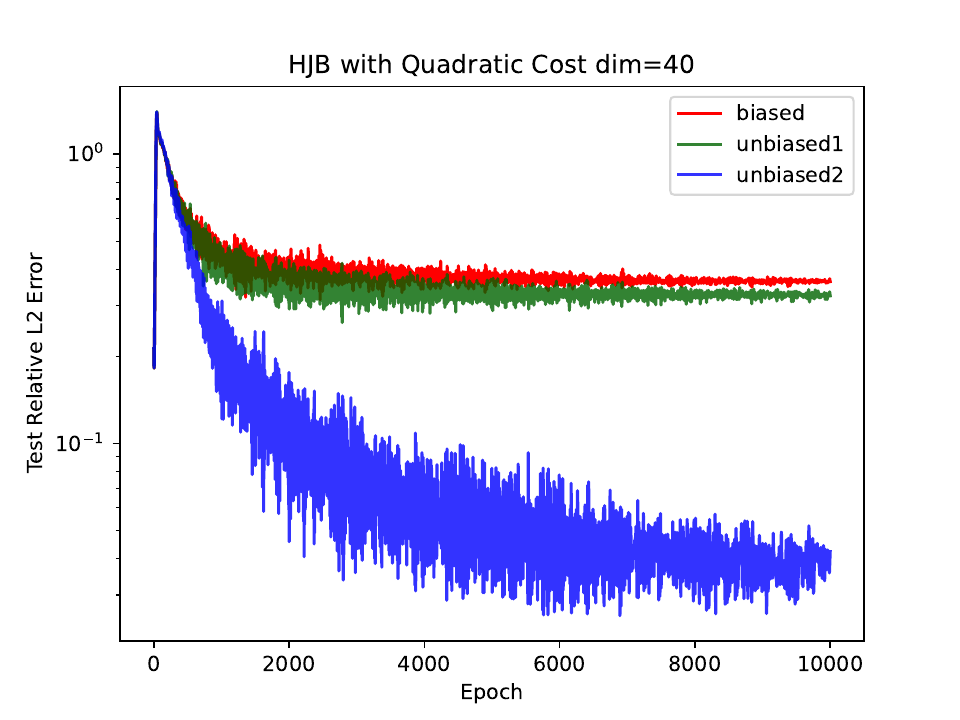}
\includegraphics[scale=0.24]{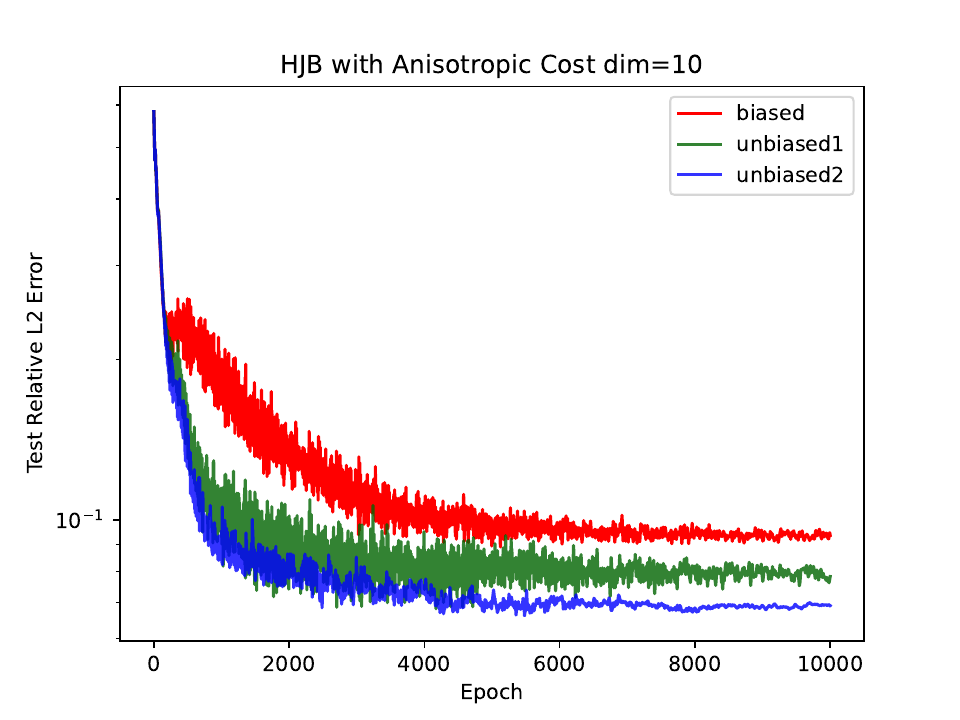}
\includegraphics[scale=0.24]{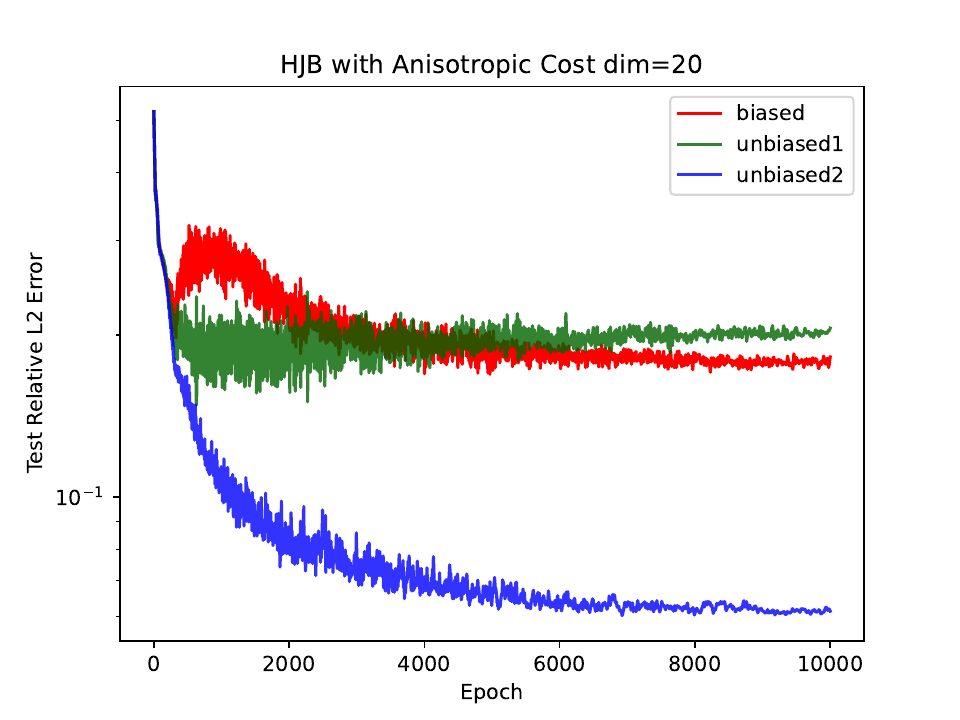}
\includegraphics[scale=0.24]{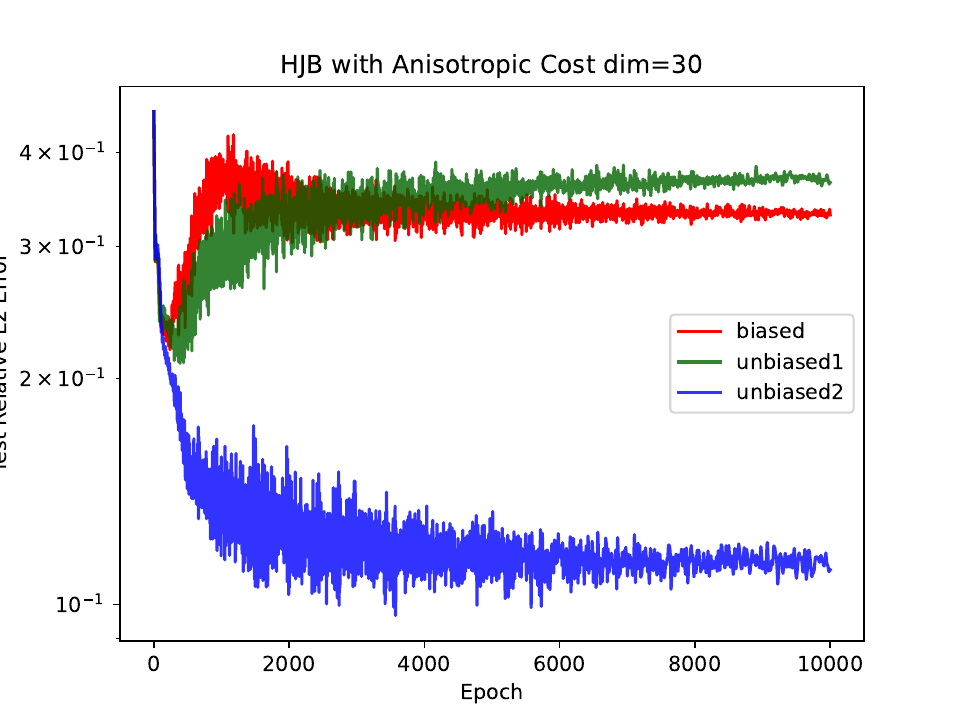}
\includegraphics[scale=0.24]{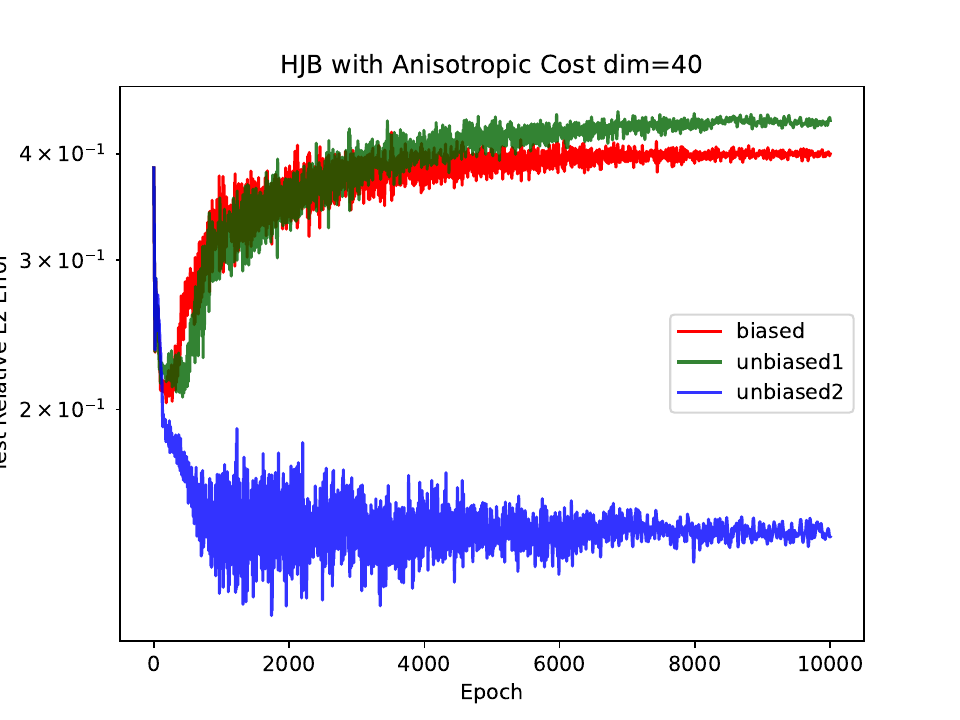}
\caption{Fisrt row: results for the HJB equation with quadratic cost. Second row: results for the HJB equation with anisotropic Rosenbrock cost.
In all dimensions of the HJB equation, the model that corrects both biases (unbiased2) performs the best. Following this, the model that corrects one bias (unbiased1) shows the next best performance, while the completely biased model performs the worst. This highlights the correctness of our analysis regarding the two biases introduced by analyzing nonlinear equations and underscores the improvement achieved by bias correction.}
\label{fig:HJB}
\end{figure}

\subsection{Viscous Burgers' PDE}
In this section, we investigate the nonlinear viscous Burgers' equation. We focus on the influence of different nonlinearities on the biases introduced to RS-PINN.

The $d$-dimensional viscous Burgers‘ equation with the initial condition at $t=0$ is given by
\begin{equation}
\begin{aligned}
& u_t + u\left(\sum_{i=1}^d\frac{\partial u(\bx)}{\partial \bx_i}\right) - \nu\left(\sum_{i=1}^d \frac{\partial^2 u(\bx)}{\partial \bx_i^2}\right) = 0, \bx \in \mathbb{R}^d, t\in [0,1].\\
& u(\bx,t=0) = \frac{1}{1+\exp{\left(\frac{\sum_{i=1}^d\bx_i}{2\nu}\right)}}.
\end{aligned}
\end{equation}
Its analytical solution is given by
\begin{equation}
u(\bx,t) = \frac{1}{1+\exp{\left(\frac{\sum_{i=1}^d\bx_i-dt/2}{2\nu}\right)}}.
\end{equation}
We choose $\nu = 0.5$ and sample points based on its corresponding SDE trajectory as before:
$t \sim \text{Unif}(0, 1), \bx \sim \mathcal{N}(t, 2 - t).$
We use the boundary augmentation given by the following model output to satisfy the initial condition automatically \cite{lu2021physics}:
\begin{align}
u_\theta(\bx, t) t + \frac{1}{1+\exp{\left(\frac{\sum_{i=1}^d\bx_i}{2\nu}\right)}}.
\end{align}
The solution of this Burgers' equation is highly complex and exhibits stiff regions where $\sum_{i=1}^d \bx_i = dt$. In these regions, the PDE solution experiences abrupt changes. This equation helps us evaluate the performance of RS-PINN on nonlinear equations with complex stiff solutions.

Here are the implementation details. The model is a 4-layer fully connected network with 128 hidden units, which is trained via Adam \cite{kingma2014adam} for 10K epochs, with an initial learning rate 1e-3, which linearly exponentially with exponent 0.9995. We select 100 random residual points at each Adam epoch and 20K fixed testing points from the SDE trajectory. The sample sizes for randomized smoothing in both training and testing are 1024 and 128, respectively, and the variance of Gaussian noise is $\sigma=1e-2$.

\begin{table}[htbp]\centering
\begin{tabular}{|c|c|c|c|c|}
\hline
Viscous Burgers' Equation & 5D & $10$D & 20D & 25D \\ \hline
Biased & 1.085E-02&	5.293E-02&	7.672E-02&	5.698E-02\\ \hline
Unbiased1 & 1.089E-03&	2.412E-03&	1.497E-02&	2.841E-02 \\ \hline
Unbiased2 & \textbf{1.030E-03} & \textbf{2.411E-03} & \textbf{1.165E-02} & \textbf{2.783E-02} \\ \hline
\end{tabular}
\caption{Results for the viscous Burgers' equation.}
\label{tab:burgers}
\end{table}

\begin{figure}[htbp]
\centering
\includegraphics[scale=0.24]{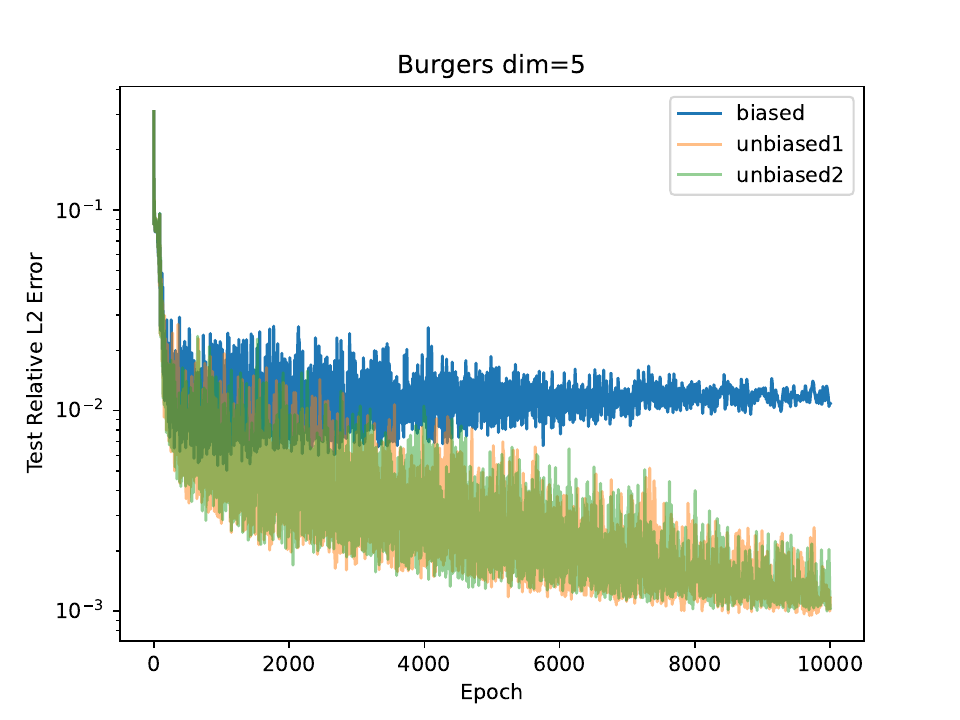}
\includegraphics[scale=0.24]{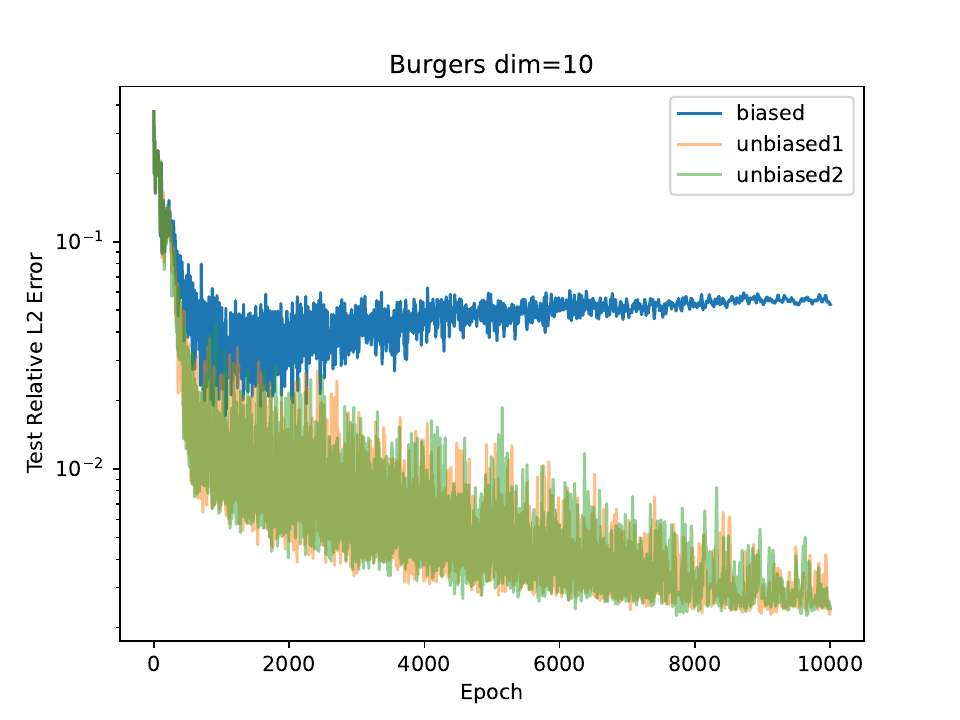}
\includegraphics[scale=0.24]{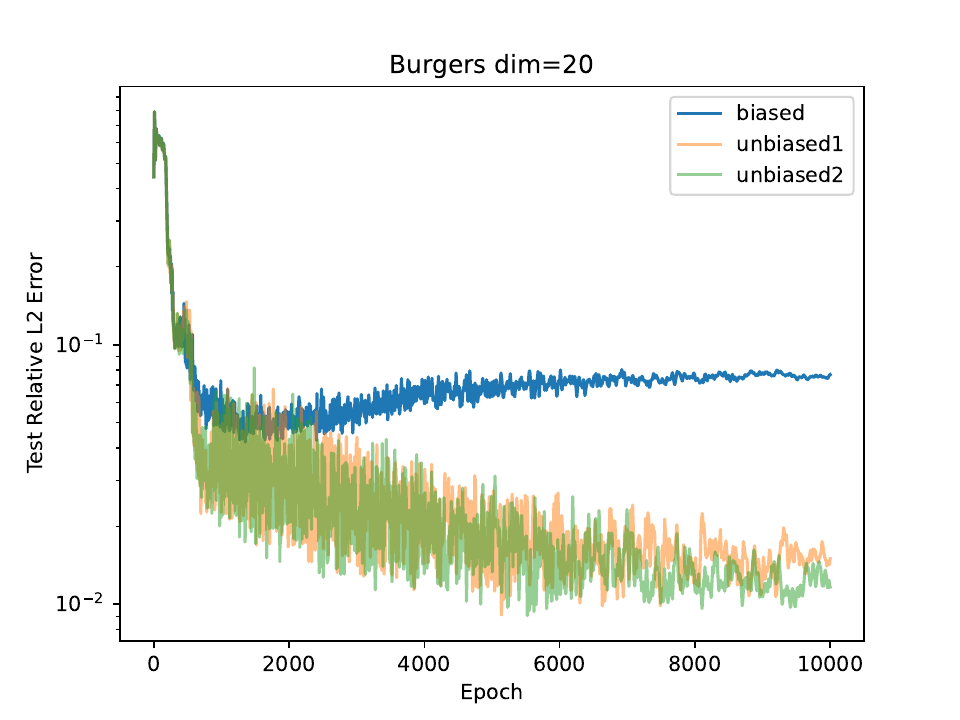}
\includegraphics[scale=0.24]{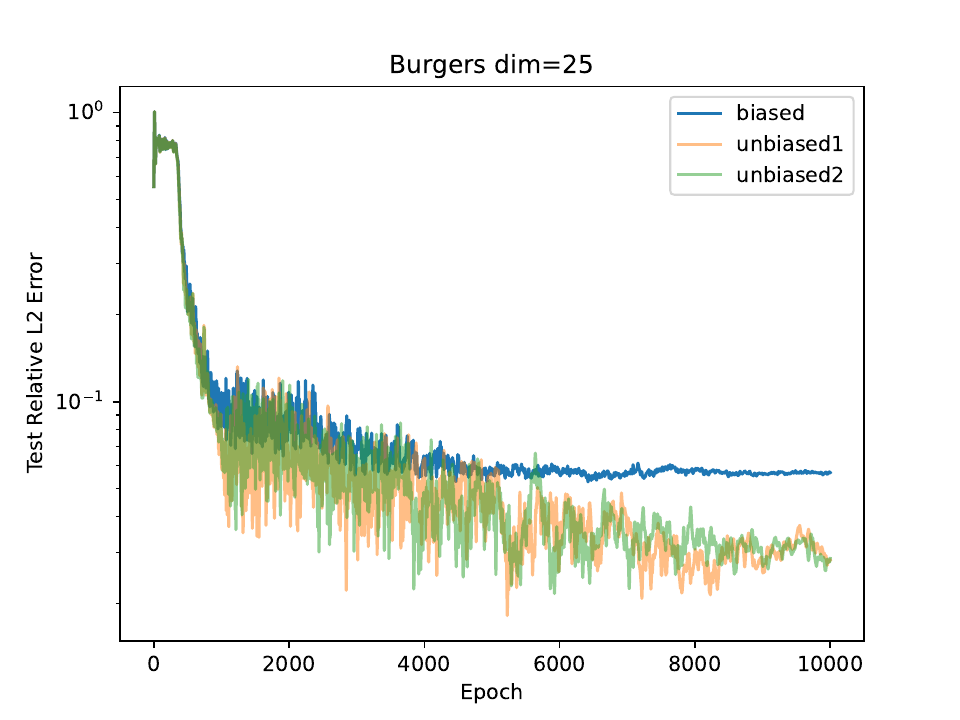}
\caption{Convergence curves for the Burgers' equation. In this example, correcting the bias from the MSE loss (unbiased1) suffices to achieve optimal performance, while the unbiased2 method can only slightly improve over the unbiased1 method.}
\label{fig:burgers}
\end{figure}

The results for the viscous Burgers' equation are shown in Table \ref{tab:burgers}, and convergence curves with respect to epoch are shown in Figure \ref{fig:burgers}. 
In all-dimensional cases in this example, unbiased1 and unbiased2 consistently outperform biased, underscoring once again the critical significance of unbiasedness for the convergence of RS-PINN. However, in contrast to the previous HJB equation case, unbiased1 alone is sufficient to achieve excellent results here, rendering unbiased2 unnecessary. Hence, the effect of bias-variance trade-off differs in various PDEs.

\subsection{Allen-Cahn and Sine-Gordon PDEs with Anisotropic Solution}
Here, we aim to consider nonseparable and anisotropic solutions for nonlinear PDEs to form complicated and nontrivial high-dimensional PDEs:
\begin{equation}
u_{\text{exact}}(\bx) = \left(1 - \Vert \bx \Vert_2^2\right)\left(\sum_{i=1}^{d-1}  c_i \sin(\bx_i +\cos(\bx_{i+1})+\bx_{i+1}\cos(\bx_i))\right),
\end{equation}
where $c_i \sim \mathcal{N}(0, 1)$.
We do not want the boundary to leak most information about the exact solution, and thus, the term $1 - \Vert \bx \Vert_2^2$ is added for a zero boundary condition.
In addition to the exact solution, the following PDEs defined within the unit ball $\mathbb{B}^d$ associated with zero boundary conditions on the unit sphere are under consideration:
\begin{itemize}
\item Allen-Cahn equation
\begin{equation}
\Delta u(\bx) + u(\bx) - u(\bx)^3 = g(\bx), \quad \bx\in \mathbb{B}^d,
\end{equation}
where $g(\bx) = \Delta u_{\text{exact}}(\bx) + u_{\text{exact}}(\bx) - u_{\text{exact}}(\bx)^3$.
\item Sine-Gordon equation
\begin{equation}
\Delta u(\bx) + \sin\left(u(\bx) \right) = g(\bx), \quad \bx \in \mathbb{B}^d,
\end{equation}
where $g(\bx) = \Delta u_{\text{exact}}(\bx) + \sin\left(u_{\text{exact}}(\bx) \right)$.
\end{itemize}
These PDEs exhibit different levels of nonlinearity. Allen-Cahn involves third-order nonlinearity, and Sine-Gordon's nonlinearity stems from the term $\sin(u)$, making it infinite-order nonlinear in theory. It is not feasible to achieve a completely unbiased version for this case. However, we will demonstrate that correcting the bias originating from the mean square error loss in the PINN loss is sufficient.

Here are the implementation details. The model is a 4-layer fully connected network with 128 hidden units, which is trained via Adam \cite{kingma2014adam} for 10K epochs, with an initial learning rate 1e-3, which linearly decays to zero at the end of the optimization. We select 100 random residual points at each Adam epoch and 20K fixed testing points uniformly from the unit ball. The sample size for randomized smoothing in both training and testing is 128, and the variance of Gaussian noise is $\sigma=1e-2$. We adopt the following model structure to satisfy the zero boundary condition with hard constraint and to avoid the boundary loss \cite{lu2021physics}:
\begin{equation}
u^{\text{RS}}_\theta(\bx) = (1 - \Vert\bx\Vert_2^2) u_\theta(\bx),
\end{equation}
where $u_\theta(\bx)$ is the randomized smoothing neural network and $u^{\text{RS}}_\theta(\bx)$ is the boundary-augmented model. We repeat our experiment 5 times with 5 independent random seeds.

\begin{table}[]
\centering
\begin{tabular}{|c|c|c|c|}
\hline
Sine-Gordon & 10D & 100D & 10,00D \\ \hline
biased & 5.712E-3 & 7.835E-3 & 6.744E-4 \\ \hline
unbiased1 & 1.410E-3 & 7.223E-3 & 6.647E-3 \\ \hline
unbiased2 & N.A. & N.A. & N.A. \\ \hline
hybrid & \textbf{1.407E-3} & \textbf{7.209E-3} & \textbf{4.732E-4} \\ \hline
\end{tabular}
\begin{tabular}{|c|c|c|c|}
\hline
Allen-Cahn & 10D & 100D & 10,00D \\ \hline
biased & 5.062E-3 & 7.923E-3 & 5.504E-4 \\ \hline
unbiased1 & 3.233E-3 & 7.298E-3 & 1.308E-3 \\ \hline
unbiased2 & 3.217E-3 & 7.293E-3 & 1.957E-2 \\ \hline
hybrid & \textbf{2.768E-3} & \textbf{7.285E-3} & \textbf{4.856E-4} \\ \hline
\end{tabular}
\caption{Results for the Sine-Gordon PDE (first row) and the Allen-Cahn PDE (second row) with anisotropic exact
solutions. Note that since the Sine-Gordon contains a sine nonlinearity, its unbiased2 version does not exist.}
\label{tab:anisotropic}
\end{table}

\begin{figure}[htbp]
\centering
\includegraphics[scale=0.3]{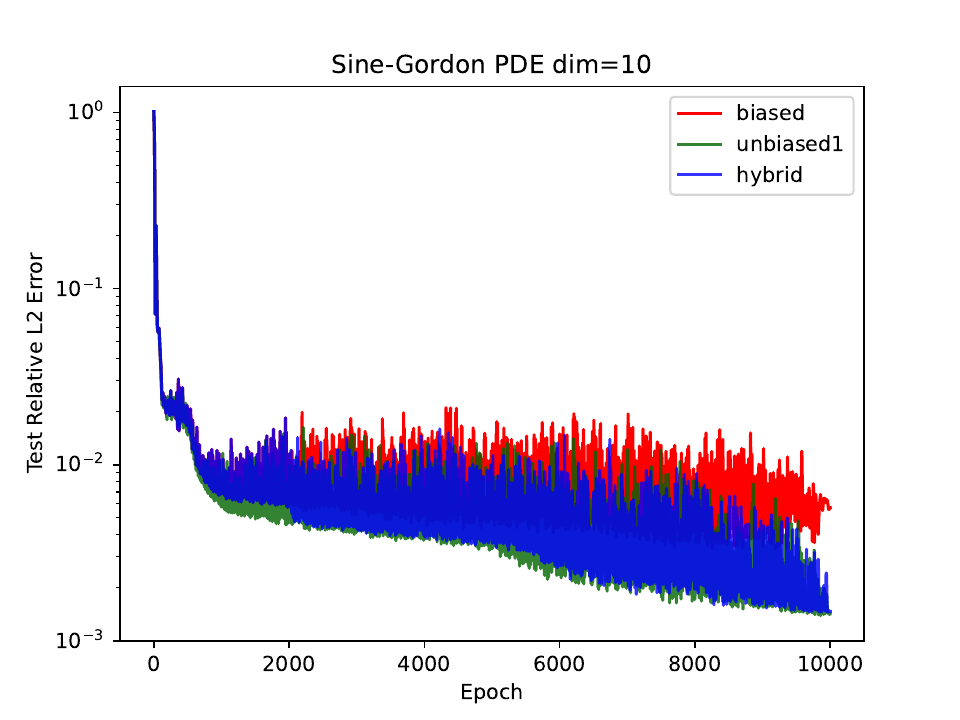}
\includegraphics[scale=0.3]{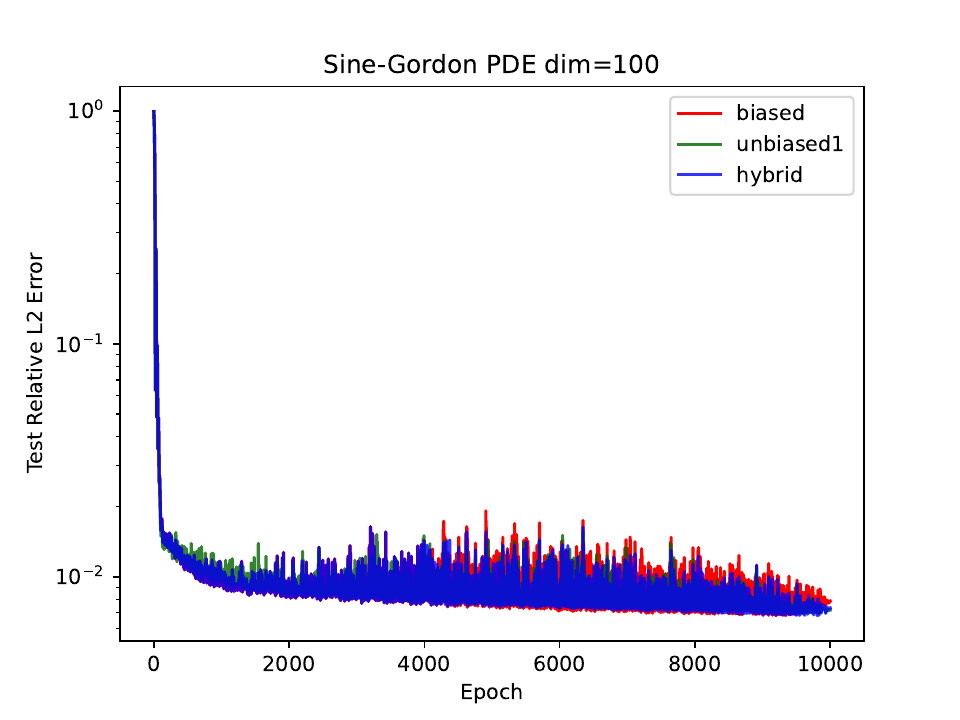}
\includegraphics[scale=0.3]{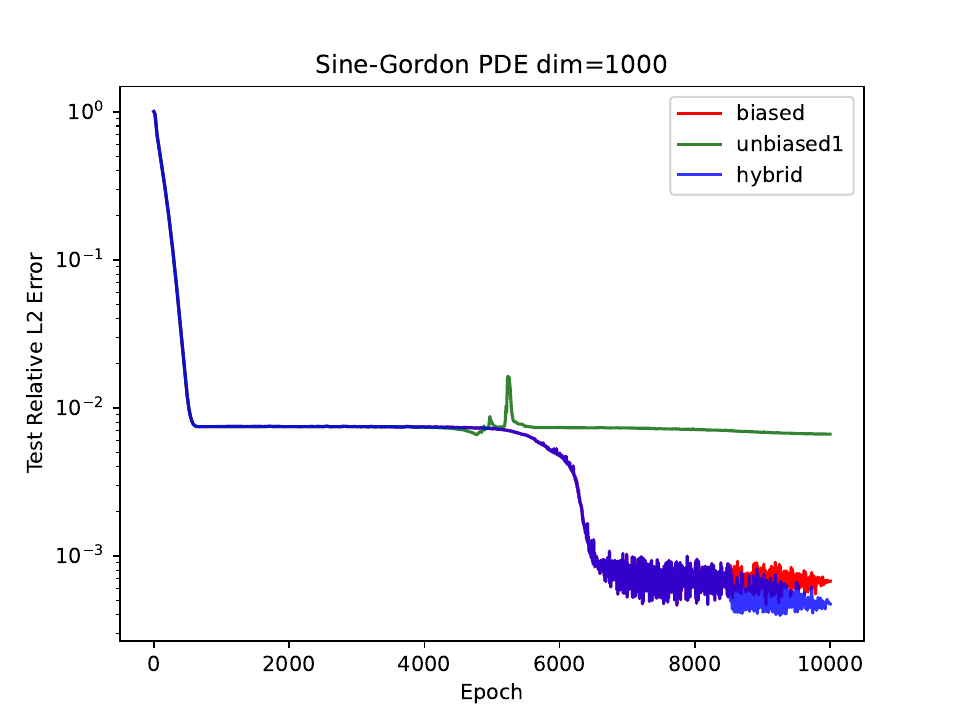}
\includegraphics[scale=0.3]{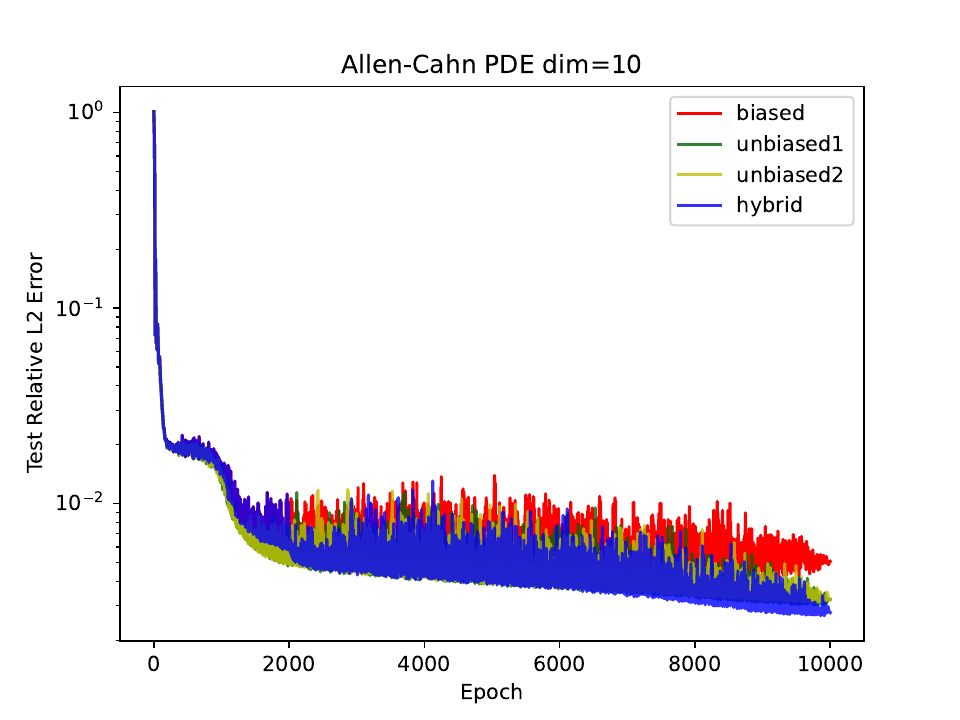}
\includegraphics[scale=0.3]{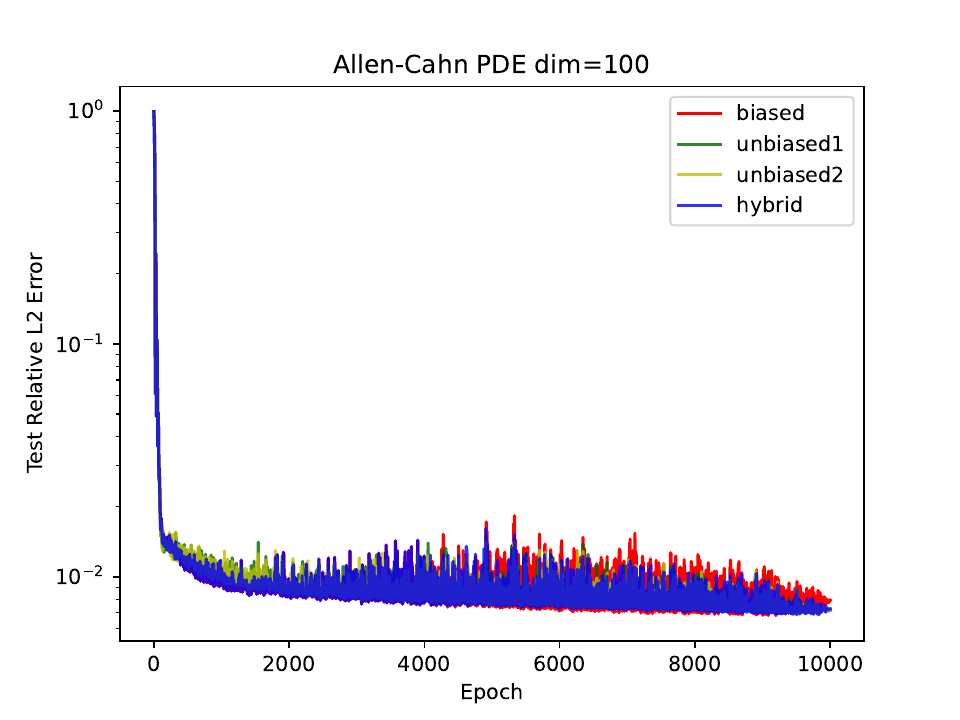}
\includegraphics[scale=0.3]{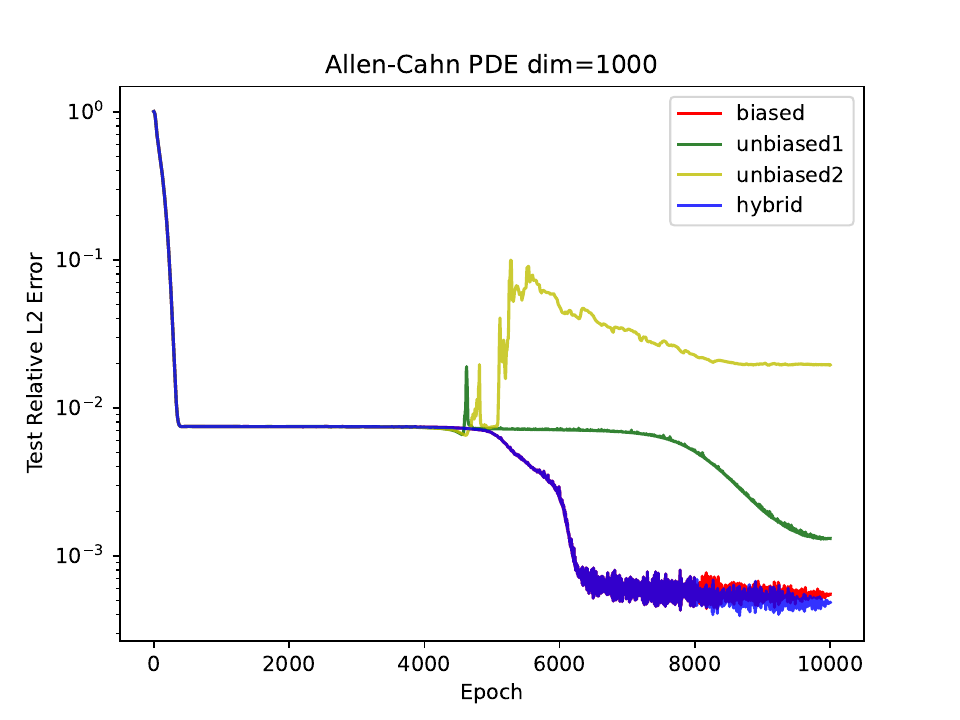}
\caption{Results for the Sine-Gordon PDE (first row) and the Allen-Cahn PDE (second row) with anisotropic exact solutions.}
\label{fig:anisotropic}
\end{figure}

The final convergence results and the convergence curves are shown in Table \ref{tab:anisotropic} and Figure \ref{fig:anisotropic}, respectively. Overall, RS-PINN is able to deal with highly complicated and anisotropic PDE solutions in high dimensions. In terms of the comparison between the biased, unbiased, and hybrid versions of RS-PINN, we obtain similar observations as the first linear Fokker-Planck equation. Concretely, 
in the relatively lower 10D scenarios, unbiased versions surpass the biased version, primarily because at lower dimensions, the smaller variance of the unbiased version makes its unbiasedness more crucial for convergence. However, in 100D, the relatively larger variance of the unbiased version balances its negative effects with the positive aspects of its unbiasedness, resulting in similar performances between unbiased and biased versions. Moving to higher dimensions, particularly in the 1000D scenario, the considerable variance of the unbiased version prevents its convergence. Additionally, in the Allen-Cahn equation, unbiased2, with its amplified variance due to increased sampling, performs worse than unbiased1 with relatively less sampling. In this context, the hybrid version, starting with the biased version to reach an acceptable solution and fine-tuning with unbiased2, achieves the optimal outcome. Among all the dimensions, the hybrid version is the most stable, since it incorporates the low variance of the biased version with the optimal convergence results by unbiased version-based fine-tuning.

\section{Summary}
We have developed an extension of Physics-Informed Neural Networks (PINNs) and their Randomized Smoothing variant (RS-PINN) to address computational challenges in high-dimensional scenarios. The identified biases in RS-PINN, originating from the nonlinearity of the Mean Squared Error (MSE) loss and the inherent nonlinearity of the PDE, have been systematically corrected using tailored techniques. The derivation of an unbiased RS-PINN has allowed us to investigate its attributes, comparing them with the biased version and paving the way for the formulation of a novel combined hybrid approach.

Our proposed bias-variance trade-off strategy, incorporating both biased and unbiased versions, introduces a novel perspective on optimizing simultaneously convergence speed and accuracy. By strategically employing the rapid convergence of the biased version in the initial stages and transitioning to the accuracy of the unbiased version for fine-tuning, we strike a balance that adapts to the dynamic nature of the optimization process.

This work contributes not only to the advancement of RS-PINN methodology but also provides a new paradigm to address biases in high-dimensional scenarios. The presented guidelines offer practical insights for navigating the complexities of biased and unbiased RS-PINN implementations. The extensive experimental validation across various high-dimensional PDEs underscores the efficacy of our bias correction techniques, reinforcing the versatility and applicability of the hybrid RS-PINN approach.

\section*{Acknowledgement}
This research of ZH, ZY, YW, and KK is partially supported by the National Research Foundation Singapore under the AI Singapore Programme (AISG Award No: AISG2-TC-2023-010-SGIL) and the Singapore Ministry of Education Academic Research Fund Tier 1 (Award No: T1 251RES2207).
The work of GEK was supported by the MURI-AFOSR FA9550-20-1-0358 projects and by the DOE SEA-CROGS project (DE-SC0023191). GEK was also supported by the ONR Vannevar Bush Faculty Fellowship (N00014-22-1-2795). 

\newpage
\appendix
\section{Proof}\label{appendix:A}
\subsection{Proof of Theorem \ref{thm:unbiased1}}
\begin{proof}
Thanks to the independence between the random variables $\delta_i$ and $\delta_i'$
\begin{equation}
\begin{aligned}
\mathbb{E}_{\delta,\delta'}\left[{L}^{(1)}_b(\theta)\right] &= 
\mathbb{E}_{\delta,\delta'}\left[\left(\frac{1}{K}\sum_{i=1}^Kf(\bx+\delta_i';\theta) - g(\bx)\right) \left(\frac{1}{K}\sum_{i=1}^Kf(\bx+\delta_i;\theta) - g(\bx)\right) \right]\\
&= 
\mathbb{E}_{\delta'}\left[\left(\frac{1}{K}\sum_{i=1}^Kf(\bx+\delta_i';\theta) - g(\bx)\right)\right] \cdot\mathbb{E}_{\delta}\left[\left(\frac{1}{K}\sum_{i=1}^Kf(\bx+\delta_i;\theta) - g(\bx)\right) \right]\\
&= 
\left(\mathbb{E}_{\delta'}\left[\frac{1}{K}\sum_{i=1}^Kf(\bx+\delta_i';\theta)\right] - g(\bx)\right) \cdot\left(\mathbb{E}_{\delta}\left[\frac{1}{K}\sum_{i=1}^Kf(\bx+\delta_i;\theta)\right]
- g(\bx)\right)\\
&= 
\left(\mathbb{E}_{\delta}\left [f(\bx+\delta;\theta)\right] - g(\bx)\right) \cdot\left(\mathbb{E}_{\delta}\left[f(\bx+\delta;\theta)\right]
- g(\bx)\right)\\
&=L_b(\theta).
\end{aligned}
\end{equation}
For their gradients with respect to $\theta$, by the chain rule
\begin{equation}
\begin{aligned}
\mathbb{E}_{\delta,\delta'}\left[\frac{\partial{L}^{(1)}_b(\theta)}{\partial \theta}\right] &= 2\mathbb{E}_{\delta,\delta'}\left[\left(\frac{1}{K}\sum_{i=1}^Kf(\bx+\delta';\theta) - g(\bx)\right) \cdot \frac{\partial}{\partial \theta}\left(\frac{1}{K}\sum_{i=1}^Kf(\bx+\delta_i;\theta) - g(\bx)\right)\right]\\
&= 2\mathbb{E}_{\delta'}\left[\left(\frac{1}{K}\sum_{i=1}^Kf(\bx+\delta_i';\theta) - g(\bx)\right)\right] \cdot\mathbb{E}_{\delta}\left[\frac{\partial}{\partial \theta}\left(\frac{1}{K}\sum_{i=1}^Kf(\bx+\delta_i;\theta) - g(\bx)\right) \right]
\\
&= 2\mathbb{E}_{\delta}\left[\left(f(\bx+\delta;\theta) - g(\bx)\right)\right] \cdot\mathbb{E}_{\delta}\left[\frac{\partial}{\partial \theta}\left(f(\bx+\delta;\theta) - g(\bx)\right) \right]
\\
&= \mathbb{E}_{\delta}\left[\frac{\partial}{\partial \theta}\left(f(\bx+\delta;\theta) - g(\bx)\right)^2 \right]
\\
& = \frac{\partial{L}_b(\theta)}{\partial \theta}.
\end{aligned}
\end{equation}
\end{proof}

\subsection{Proof of Theorem \ref{thm:unbiased2}}
\begin{proof}
\begin{equation}
\begin{aligned}
\mathbb{E}_{\delta,\delta',\delta'',\delta'''}\left[{L}^{(2)}_r(\theta)\right] &= \left(\left\langle\mathbb{E}_{\delta}\left[\frac{1}{K}\sum_{i=1}^K\frac{\delta_i}{\sigma^2}f(\bx+\delta_i;\theta)\right], \mathbb{E}_{\delta'}\left[\frac{1}{K}\sum_{i=1}^K\frac{\delta_i'}{\sigma^2}f(\bx+\delta_i';\theta)\right]\right\rangle-g(\bx)\right)\times\\
&\quad\left(\left\langle\mathbb{E}_{\delta''}\left[\frac{1}{K}\sum_{i=1}^K\frac{\delta_i''}{\sigma^2}f(\bx+\delta_i'';\theta)\right], \mathbb{E}_{\delta'''}\left[\frac{1}{K}\sum_{i=1}^K\frac{\delta_i'''}{\sigma^2}f(\bx+\delta_i''';\theta)\right]\right\rangle-g(\bx)\right) \\
&= \left(\left\langle\mathbb{E}_{\delta}\left[\frac{\delta}{\sigma^2}f(\bx+\delta;\theta)\right], \mathbb{E}_{\delta}\left[\frac{\delta}{\sigma^2}f(\bx+\delta;\theta)\right]\right\rangle-g(\bx)\right)\times\\
&\quad\left(\left\langle\mathbb{E}_{\delta}\left[\frac{\delta}{\sigma^2}f(\bx+\delta;\theta)\right], \mathbb{E}_{\delta}\left[\frac{\delta}{\sigma^2}f(\bx+\delta;\theta)\right]\right\rangle-g(\bx)\right) \\
&= L_r(\theta).
\end{aligned}
\end{equation}
\begin{equation}
\begin{aligned}
\mathbb{E}_{\delta,\delta',\delta'',\delta'''}\left[\frac{\partial}{\partial \theta}{L}^{(2)}_r(\theta)\right] &= 2\left(\left\langle\mathbb{E}_{\delta}\left[\frac{1}{K}\sum_{i=1}^K\frac{\delta_i}{\sigma^2}f(\bx+\delta_i;\theta)\right], \mathbb{E}_{\delta'}\left[\frac{1}{K}\sum_{i=1}^K\frac{\delta_i'}{\sigma^2}f(\bx+\delta_i';\theta)\right]\right\rangle-g(\bx)\right)\times\\
&\quad\frac{\partial}{\partial\theta}\left(\left\langle\mathbb{E}_{\delta''}\left[\frac{1}{K}\sum_{i=1}^K\frac{\delta_i''}{\sigma^2}f(\bx+\delta_i'';\theta)\right], \mathbb{E}_{\delta'''}\left[\frac{1}{K}\sum_{i=1}^K\frac{\delta_i'''}{\sigma^2}f(\bx+\delta_i''';\theta)\right]\right\rangle-g(\bx)\right) \\
&= 2\left(\left\langle\mathbb{E}_{\delta}\left[\frac{\delta}{\sigma^2}f(\bx+\delta;\theta)\right], \mathbb{E}_{\delta}\left[\frac{\delta}{\sigma^2}f(\bx+\delta;\theta)\right]\right\rangle-g(\bx)\right)\times\\
&\quad\frac{\partial}{\partial\theta}\left(\left\langle\mathbb{E}_{\delta}\left[\frac{\delta}{\sigma^2}f(\bx+\delta;\theta)\right], \mathbb{E}_{\delta}\left[\frac{\delta}{\sigma^2}f(\bx+\delta;\theta)\right]\right\rangle-g(\bx)\right) \\
&=
\frac{\partial}{\partial \theta}L_r(\theta).
\end{aligned}
\end{equation}
\end{proof}

\section{Detailed Loss Function}\label{appendix:B}
\subsection{Hamilton-Jacobi-Bellman PDE}
The previous HJB equation given by
\begin{equation}
u_t = \Delta_{\bx} u - \Vert\nabla_{\bx} u(\bx)\Vert^2.
\end{equation}
which has a nonlinearity order of two due to the $\Vert \nabla_{\bx} u \Vert^2$ term.

As before, to simplify the discussion, we assume the residual condition is $g(\bx)$, and we ignore the linear term of the HJB equation and only consider the nonlinear term:
$
\Vert\nabla_{\bx} u(\bx)\Vert^2.
$
Recall that
$
\nabla_{\bx} u(\bx; \theta) = \mathbb{E}_{\delta \sim \mathcal{N}(0, \sigma^2\boldsymbol{I})}\left[\frac{\delta}{\sigma^2}f(\bx+\delta;\theta)\right].
$

The true residual loss on a residual point $\bx$ is:
\begin{equation}
L_r(\theta) = \left(\Vert\nabla_{\bx}u(\bx; \theta) \Vert^2 - g(\bx)\right)^2 
\end{equation}

The biased loss from He et al. \cite{he2023learning} is
\begin{equation}
{L}^{(0)}_r(\theta) = \left(\left\|\frac{1}{K}\sum_{i=1}^K\frac{\delta_i}{\sigma^2}f(\bx+\delta_i;\theta)\right\|^2 - g(\bx)\right)^2. 
\end{equation}

The unbiased1 loss by correcting the bias from the nonlinear MSE loss solely is
\begin{equation}
{L}^{(1)}_r(\theta) = \left(\left\|\frac{1}{K}\sum_{i=1}^K\frac{\delta_i}{\sigma^2}f(\bx+\delta_i;\theta)\right\|^2-g(\bx)\right)\times\left(\left\|\frac{1}{K}\sum_{i=1}^K\frac{\delta_i'}{\sigma^2}f(\bx+\delta_i';\theta)\right\|^2-g(\bx)\right). 
\end{equation}

The unbiased2 loss by correcting the biases from the MSE loss and the PDE nonlinearity is
\begin{equation}
\begin{aligned}
{L}^{(2)}_r(\theta) &= \left(\left\langle\frac{1}{K}\sum_{i=1}^K\frac{\delta_i}{\sigma^2}f(\bx+\delta_i;\theta), \frac{1}{K}\sum_{i=1}^K\frac{\delta_i'}{\sigma^2}f(\bx+\delta_i';\theta)\right\rangle-g(\bx)\right)\times\\
&\quad\left(\left\langle\frac{1}{K}\sum_{i=1}^K\frac{\delta_i''}{\sigma^2}f(\bx+\delta_i'';\theta), \frac{1}{K}\sum_{i=1}^K\frac{\delta_i'''}{\sigma^2}f(\bx+\delta_i''';\theta)\right\rangle-g(\bx)\right). 
\end{aligned}
\end{equation}

\subsection{Allen-Cahn PDE}
For the Allen-Cahn (AC) PDE given by
\begin{equation}
u_t = \Delta_{\bx} u + u - u^3,
\end{equation}
its nonlinearity stems from the term $u^3$, which is a cubic function. Therefore, the nonlinearity order of the AC equation is three.
During the model training, we are required to sample the three $u$ terms in $u^3 = u \cdot u \cdot u$ independently for unbiased gradients.

As before, to simplify the discussion, we assume the residual condition is $g(\bx)$, and we ignore the linear term of the HJB equation and only consider the nonlinear term:
$
u^3,
$.
Recall that
$
u(\bx; \theta) = \mathbb{E}_{\delta \sim \mathcal{N}(0, \sigma^2\boldsymbol{I})}\left[f(\bx+\delta;\theta)\right].
$

The true residual loss on a residual point $\bx$ is:
\begin{equation}
L_r(\theta) = \left(u(\bx; \theta)^3 - g(\bx)\right)^2 
\end{equation}

The biased loss from He et al. \cite{he2023learning} is
\begin{equation}
{L}^{(0)}_r(\theta) = \left(\left(\frac{1}{K}\sum_{i=1}^Kf(\bx+\delta_i;\theta)\right)^3 - g(\bx)\right)^2. 
\end{equation}

The unbiased1 loss by correcting the bias from the nonlinear MSE loss solely is
\begin{equation}
{L}^{(1)}_r(\theta) = \left(\left(\frac{1}{K}\sum_{i=1}^Kf(\bx+\delta_i;\theta)\right)^3-g(\bx)\right)\times\left(\left(\frac{1}{K}\sum_{i=1}^Kf(\bx+\delta_i';\theta)\right)^3-g(\bx)\right). 
\end{equation}

The unbiased2 loss by correcting the biases from the MSE loss and the PDE nonlinearity is
\begin{equation}
\begin{aligned}
{L}^{(2)}_r(\theta) &= \left(\left(\frac{1}{K}\sum_{i=1}^Kf(\bx+\delta_i^{(1)};\theta)\right)\left(\frac{1}{K}\sum_{i=1}^Kf(\bx+\delta_i^{(2)};\theta)\right)\left(\frac{1}{K}\sum_{i=1}^Kf(\bx+\delta_i^{(3)};\theta)\right)-g(\bx)\right)\times\\
&\quad\left(\left(\frac{1}{K}\sum_{i=1}^Kf(\bx+\delta_i^{(4)};\theta)\right)\left(\frac{1}{K}\sum_{i=1}^Kf(\bx+\delta_i^{(5)};\theta)\right)\left(\frac{1}{K}\sum_{i=1}^Kf(\bx+\delta_i^{(6)};\theta)\right)-g(\bx)\right),
\end{aligned}
\end{equation}
where $\delta_i^{(1)},\delta_i^{(2)},\delta_i^{(3)},\delta_i^{(4)},\delta_i^{(5)},\delta_i^{(6)}$ are six independent groups of Gaussian random samples.

\subsection{Viscous Burgers' PDE}
For the viscous Burgers' equation:
\begin{equation}
\begin{aligned}
& u_t + u\left(\sum_{i=1}^d\frac{\partial u(\bx)}{\partial \bx_i}\right) - \nu\left(\sum_{i=1}^d \frac{\partial^2 u(\bx)}{\partial \bx_i^2}\right) = 0, \bx \in \mathbb{R}^d, t\in [0,1],
\end{aligned}
\end{equation}
its nonlinearity stems from the term $u\sum_{i=1}^du_{\bx_i}$. Therefore, the nonlinearity order of the viscous Burgers' PDE is three.
During the model training, we are required to sample the $u$ and $\nabla_{\bx} u$ independently for unbiased gradients.

As before, to simplify the discussion, we assume the residual condition is $g(\bx)$, and we ignore the linear term of the HJB equation and only consider the nonlinear term:
$u\sum_{i=1}^du_{\bx_i}$. Furthermore, we will use the operator $\operatorname{sum}$ to denote the element-wise sum of a vector.

The true residual loss on a residual point $\bx$ is:
\begin{equation}
L_r(\theta) = \left(u(\bx;\theta)\sum_{i=1}^d\frac{\partial}{\partial\bx_i}u(\bx;\theta) - g(\bx)\right)^2 
\end{equation}

The biased loss from He et al. \cite{he2023learning} is
\begin{equation}
{L}^{(0)}_r(\theta) = \left(\left(\frac{1}{K}\sum_{i=1}^Kf(\bx+\delta_i;\theta)\right)\operatorname{sum}\left(\frac{1}{K}\sum_{i=1}^K\frac{\delta_i}{\sigma^2}f(\bx+\delta_i;\theta)\right) - g(\bx)\right)^2 .
\end{equation}

The unbiased1 loss by correcting the bias from the nonlinear MSE loss solely is
\begin{equation}
\begin{aligned}
{L}^{(1)}_r(\theta) &= \left(\left(\frac{1}{K}\sum_{i=1}^Kf(\bx+\delta_i;\theta)\right)\operatorname{sum}\left(\frac{1}{K}\sum_{i=1}^K\frac{\delta_i}{\sigma^2}f(\bx+\delta_i;\theta)\right) - g(\bx)\right) \times\\
&\quad \left(\left(\frac{1}{K}\sum_{i=1}^Kf(\bx+\delta_i';\theta)\right)\operatorname{sum}\left(\frac{1}{K}\sum_{i=1}^K\frac{\delta_i'}{\sigma^2}f(\bx+\delta_i';\theta)\right) - g(\bx)\right).
\end{aligned}
\end{equation}

The unbiased2 loss by correcting the biases from the MSE loss and the PDE nonlinearity is
\begin{equation}
\begin{aligned}
{L}^{(2)}_r(\theta) &= \left(\left(\frac{1}{K}\sum_{i=1}^Kf(\bx+\delta_i;\theta)\right)\operatorname{sum}\left(\frac{1}{K}\sum_{i=1}^K\frac{\delta_i'}{\sigma^2}f(\bx+\delta_i';\theta)\right) - g(\bx)\right) \times\\
&\quad \left(\left(\frac{1}{K}\sum_{i=1}^Kf(\bx+\delta_i'';\theta)\right)\operatorname{sum}\left(\frac{1}{K}\sum_{i=1}^K\frac{\delta_i'''}{\sigma^2}f(\bx+\delta_i''';\theta)\right) - g(\bx)\right).
\end{aligned}
\end{equation}

\subsection{Sine-Gordon PDE}
However, our method cannot deal with nonlinear like $\sin (u)$ in the Sine-Gorden PDE:
\begin{equation}
u_t = \Delta_{\bx} u + \sin (u)
\end{equation}
However, we can increase the sample size $K$ in the Monte Carlo to minimize the bias.
Fortunately, correcting the forward and backward bias does suffice for the Sine-Gordon equation, i.e., we can still correct the bias from the forward and backward passes to improve over the original formulation in He et al. \cite{he2023learning}, which is actually sufficient to obtain a low error in high dimensions.

As before, to simplify the discussion, we assume the residual condition is $g(\bx)$, and we ignore the linear term of the HJB equation and only consider the nonlinear term:
$
\sin \left(u(\bx)\right).
$

The true residual loss on a residual point $\bx$ is:
\begin{equation}
L_r(\theta) = \left(\sin \left(u(\bx; \theta)\right) - g(\bx)\right)^2 
\end{equation}

The biased loss from He et al. \cite{he2023learning} is
\begin{equation}
{L}^{(0)}_r(\theta) = \left(\sin\left(\frac{1}{K}\sum_{i=1}^Kf(\bx+\delta_i;\theta)\right) - g(\bx)\right)^2 
\end{equation}

The unbiased1 loss by correcting the bias from the nonlinear MSE loss solely is
\begin{equation}
{L}^{(1)}_r(\theta) =  \left(\sin\left(\frac{1}{K}\sum_{i=1}^Kf(\bx+\delta_i;\theta)\right) - g(\bx)\right)\times \left(\sin\left(\frac{1}{K}\sum_{i=1}^Kf(\bx+\delta_i';\theta)\right) - g(\bx)\right) .
\end{equation}

The unbiased2 loss that corrects the two biases does not exist for this equation due to the sine nonlinearity.

\newpage
\bibliographystyle{plain}
\bibliography{main}

\begin{thebibliography}{10}

\bibitem{beck2021deep}
Christian Beck, Sebastian Becker, Patrick Cheridito, Arnulf Jentzen, and Ariel
  Neufeld.
\newblock Deep splitting method for parabolic pdes.
\newblock {\em SIAM Journal on Scientific Computing}, 43(5):A3135--A3154, 2021.

\bibitem{beck2019machine}
Christian Beck, Weinan E, and Arnulf Jentzen.
\newblock Machine learning approximation algorithms for high-dimensional fully
  nonlinear partial differential equations and second-order backward stochastic
  differential equations.
\newblock {\em Journal of Nonlinear Science}, 29:1563--1619, 2019.

\bibitem{beck2020overcoming}
Christian Beck, Lukas Gonon, and Arnulf Jentzen.
\newblock Overcoming the curse of dimensionality in the numerical approximation
  of high-dimensional semilinear elliptic partial differential equations.
\newblock {\em arXiv preprint arXiv:2003.00596}, 2020.

\bibitem{beck2020overcoming_ac}
Christian Beck, Fabian Hornung, Martin Hutzenthaler, Arnulf Jentzen, and Thomas
  Kruse.
\newblock Overcoming the curse of dimensionality in the numerical approximation
  of allen--cahn partial differential equations via truncated full-history
  recursive multilevel picard approximations.
\newblock {\em Journal of Numerical Mathematics}, 28(4):197--222, 2020.

\bibitem{becker2020numerical}
Sebastian Becker, Ramon Braunwarth, Martin Hutzenthaler, Arnulf Jentzen, and
  Philippe von Wurstemberger.
\newblock Numerical simulations for full history recursive multilevel picard
  approximations for systems of high-dimensional partial differential
  equations.
\newblock {\em arXiv preprint arXiv:2005.10206}, 2020.

\bibitem{bettencourt2019taylormode}
Jesse Bettencourt, Matthew~J. Johnson, and David Duvenaud.
\newblock Taylor-mode automatic differentiation for higher-order derivatives in
  {JAX}.
\newblock In {\em Program Transformations for ML Workshop at NeurIPS 2019},
  2019.

\bibitem{chan2019machine}
Quentin Chan-Wai-Nam, Joseph Mikael, and Xavier Warin.
\newblock Machine learning for semi linear pdes.
\newblock {\em Journal of scientific computing}, 79(3):1667--1712, 2019.

\bibitem{CHIU2022114909}
Pao-Hsiung Chiu, Jian~Cheng Wong, Chinchun Ooi, My~Ha Dao, and Yew-Soon Ong.
\newblock Can-pinn: A fast physics-informed neural network based on
  coupled-automatic–numerical differentiation method.
\newblock {\em Computer Methods in Applied Mechanics and Engineering},
  395:114909, 2022.

\bibitem{cho2022separable}
Junwoo Cho, Seungtae Nam, Hyunmo Yang, Seok-Bae Yun, Youngjoon Hong, and
  Eunbyung Park.
\newblock Separable pinn: Mitigating the curse of dimensionality in
  physics-informed neural networks.
\newblock {\em arXiv preprint arXiv:2211.08761}, 2022.

\bibitem{pmlr-v97-cohen19c}
Jeremy Cohen, Elan Rosenfeld, and Zico Kolter.
\newblock Certified adversarial robustness via randomized smoothing.
\newblock In Kamalika Chaudhuri and Ruslan Salakhutdinov, editors, {\em
  Proceedings of the 36th International Conference on Machine Learning},
  volume~97 of {\em Proceedings of Machine Learning Research}, pages
  1310--1320. PMLR, 09--15 Jun 2019.

\bibitem{feng2023does}
Haozhe Feng, Tianyu Pang, Chao Du, Wei Chen, Shuicheng Yan, and Min Lin.
\newblock Does federated learning really need backpropagation?
\newblock {\em arXiv preprint arXiv:2301.12195}, 2023.

\bibitem{han2018solving}
Jiequn Han, Arnulf Jentzen, and Weinan E.
\newblock Solving high-dimensional partial differential equations using deep
  learning.
\newblock {\em Proceedings of the National Academy of Sciences},
  115(34):8505--8510, 2018.

\bibitem{han2017deep}
Jiequn Han, Arnulf Jentzen, et~al.
\newblock Deep learning-based numerical methods for high-dimensional parabolic
  partial differential equations and backward stochastic differential
  equations.
\newblock {\em Communications in mathematics and statistics}, 5(4):349--380,
  2017.

\bibitem{he2023learning}
Di~He, Shanda Li, Wenlei Shi, Xiaotian Gao, Jia Zhang, Jiang Bian, Liwei Wang,
  and Tie-Yan Liu.
\newblock Learning physics-informed neural networks without stacked
  back-propagation.
\newblock In {\em International Conference on Artificial Intelligence and
  Statistics}, pages 3034--3047. PMLR, 2023.

\bibitem{henry2017deep}
Pierre Henry-Labordere.
\newblock Deep primal-dual algorithm for bsdes: Applications of machine
  learning to cva and im.
\newblock {\em Available at SSRN 3071506}, 2017.

\bibitem{hu2021extended}
Zheyuan Hu, Ameya~D. Jagtap, George~Em Karniadakis, and Kenji Kawaguchi.
\newblock When do extended physics-informed neural networks (xpinns) improve
  generalization?
\newblock {\em SIAM Journal on Scientific Computing}, 44(5):A3158--A3182, 2022.

\bibitem{hu2023tackling}
Zheyuan Hu, Khemraj Shukla, George~Em Karniadakis, and Kenji Kawaguchi.
\newblock Tackling the curse of dimensionality with physics-informed neural
  networks.
\newblock {\em arXiv preprint arXiv:2307.12306}, 2023.

\bibitem{hure2020deep}
C{\^o}me Hur{\'e}, Huy{\^e}n Pham, and Xavier Warin.
\newblock Deep backward schemes for high-dimensional nonlinear pdes.
\newblock {\em Mathematics of Computation}, 89(324):1547--1579, 2020.

\bibitem{hutzenthaler2020overcoming}
Martin Hutzenthaler, Arnulf Jentzen, Thomas Kruse, Tuan Anh~Nguyen, and
  Philippe von Wurstemberger.
\newblock Overcoming the curse of dimensionality in the numerical approximation
  of semilinear parabolic partial differential equations.
\newblock {\em Proceedings of the Royal Society A}, 476(2244):20190630, 2020.

\bibitem{hutzenthaler2021multilevel}
Martin Hutzenthaler, Arnulf Jentzen, Thomas Kruse, et~al.
\newblock Multilevel picard iterations for solving smooth semilinear parabolic
  heat equations.
\newblock {\em Partial Differential Equations and Applications}, 2(6):1--31,
  2021.

\bibitem{jagtap2020adaptive}
Ameya~D Jagtap, Kenji Kawaguchi, and George~Em Karniadakis.
\newblock Adaptive activation functions accelerate convergence in deep and
  physics-informed neural networks.
\newblock {\em Journal of Computational Physics}, 404:109136, 2020.

\bibitem{ji2020three}
Shaolin Ji, Shige Peng, Ying Peng, and Xichuan Zhang.
\newblock Three algorithms for solving high-dimensional fully coupled fbsdes
  through deep learning.
\newblock {\em IEEE Intelligent Systems}, 35(3):71--84, 2020.

\bibitem{karniadakis2021physics}
George~Em Karniadakis, Ioannis~G Kevrekidis, Lu~Lu, Paris Perdikaris, Sifan
  Wang, and Liu Yang.
\newblock Physics-informed machine learning.
\newblock {\em Nature Reviews Physics}, 3(6):422--440, 2021.

\bibitem{kawaguchi2016deep}
Kenji Kawaguchi.
\newblock Deep learning without poor local minima.
\newblock In {\em Advances in neural information processing systems (NeurIPS)},
  pages 586--594, 2016.

\bibitem{icml2023kzxinfodl}
Kenji Kawaguchi, Zhun Deng, Xu~Ji, and Jiaoyang Huang.
\newblock How does information bottleneck help deep learning?
\newblock In {\em International Conference on Machine Learning (ICML)}, 2023.

\bibitem{kawaguchi2017generalization}
Kenji Kawaguchi, Leslie~Pack Kaelbling, and Yoshua Bengio.
\newblock Generalization in deep learning.
\newblock {\em Cambridge University Press}, 2022.

\bibitem{kingma2014adam}
Diederik~P Kingma and Jimmy Ba.
\newblock Adam: A method for stochastic optimization.
\newblock {\em ICLR}, 2015.

\bibitem{lecuyer2019certified}
Mathias Lecuyer, Vaggelis Atlidakis, Roxana Geambasu, Daniel Hsu, and Suman
  Jana.
\newblock Certified robustness to adversarial examples with differential
  privacy.
\newblock In {\em 2019 IEEE symposium on security and privacy (SP)}, pages
  656--672. IEEE, 2019.

\bibitem{lu2021physics}
Lu~Lu, Raphael Pestourie, Wenjie Yao, Zhicheng Wang, Francesc Verdugo, and
  Steven~G Johnson.
\newblock Physics-informed neural networks with hard constraints for inverse
  design.
\newblock {\em SIAM Journal on Scientific Computing}, 43(6):B1105--B1132, 2021.

\bibitem{lv2021hybrid}
Chunyue Lv, Lei Wang, and Chenming Xie.
\newblock A hybrid physics-informed neural network for nonlinear partial
  differential equation.
\newblock {\em arXiv preprint arXiv:2112.01696}, 2021.

\bibitem{mishra2020estimates}
Siddhartha Mishra and Roberto Molinaro.
\newblock Estimates on the generalization error of physics informed neural
  networks (pinns) for approximating pdes.
\newblock {\em arXiv preprint arXiv:2006.16144}, 2020.

\bibitem{pang2019fpinns}
Guofei Pang, Lu~Lu, and George~Em Karniadakis.
\newblock fpinns: Fractional physics-informed neural networks.
\newblock {\em SIAM Journal on Scientific Computing}, 41(4):A2603--A2626, 2019.

\bibitem{raissi2018forward}
Maziar Raissi.
\newblock Forward-backward stochastic neural networks: Deep learning of
  high-dimensional partial differential equations.
\newblock {\em arXiv preprint arXiv:1804.07010}, 2018.

\bibitem{raissi2019physics}
Maziar Raissi, Paris Perdikaris, and George~E Karniadakis.
\newblock Physics-informed neural networks: A deep learning framework for
  solving forward and inverse problems involving nonlinear partial differential
  equations.
\newblock {\em Journal of Computational Physics}, 378:686--707, 2019.

\bibitem{shin2020convergence}
Yeonjong Shin, Jerome Darbon, and George~Em Karniadakis.
\newblock On the convergence of physics informed neural networks for linear
  second-order elliptic and parabolic type pdes.
\newblock {\em arXiv preprint arXiv:2004.01806}, 2020.

\bibitem{sirignano2018dgm}
Justin Sirignano and Konstantinos Spiliopoulos.
\newblock Dgm: A deep learning algorithm for solving partial differential
  equations.
\newblock {\em Journal of computational physics}, 375:1339--1364, 2018.

\bibitem{wang20222}
Chuwei Wang, Shanda Li, Di~He, and Liwei Wang.
\newblock Is \$l{\textasciicircum}2\$ physics informed loss always suitable for
  training physics informed neural network?
\newblock In Alice~H. Oh, Alekh Agarwal, Danielle Belgrave, and Kyunghyun Cho,
  editors, {\em Advances in Neural Information Processing Systems}, 2022.

\bibitem{wang2022tensor}
Yifan Wang, Pengzhan Jin, and Hehu Xie.
\newblock Tensor neural network and its numerical integration.
\newblock {\em arXiv preprint arXiv:2207.02754}, 2022.

\bibitem{wang2022solving}
Yifan Wang, Yangfei Liao, and Hehu Xie.
\newblock Solving schr$\backslash$"$\{$o$\}$ dinger equation using tensor
  neural network.
\newblock {\em arXiv preprint arXiv:2209.12572}, 2022.

\end{thebibliography}

\end{document}